\documentclass[11pt]{article}

\usepackage{arxiv}
\usepackage{enumitem}

\usepackage[numbers,compress,sort]{natbib}

\usepackage{algorithm,algorithmic}

\usepackage{adjustbox}

\usepackage{array, makecell}
\usepackage{comment}
\usepackage{svg}

\usepackage[utf8]{inputenc} % allow utf-8 input
\usepackage[T1]{fontenc}    % use 8-bit T1 fonts
\usepackage{url}            % simple URL typesetting
\usepackage{booktabs}       % professional-quality tables
\usepackage{amsfonts}       % blackboard math symbols
\usepackage{nicefrac}       % compact symbols for 1/2, etc.
\usepackage{microtype}      % microtypography
\usepackage{xcolor}         % colors
\usepackage{amsthm}
\usepackage{amsmath}

\usepackage{float}
\usepackage{mdframed}
\usepackage{caption}
\usepackage{newfloat}

\usepackage{graphicx}
\usepackage{subfigure}
\usepackage{caption}
\usepackage{subcaption}

\usepackage[utf8]{inputenc} % allow utf-8 input
\usepackage[T1]{fontenc}    % use 8-bit T1 fonts
\usepackage{url}            % simple URL typesetting
\usepackage{booktabs}       % professional-quality tables
\usepackage{amsfonts}       % blackboard math symbols
\usepackage{nicefrac}       % compact symbols for 1/2, etc.
\usepackage{microtype}      % microtypography
\usepackage{xcolor} 

\usepackage{amsmath}
\usepackage{amsthm}

\usepackage{tabularx}

\usepackage{comment} 
\usepackage{natbib}

\newcommand{\X}{\mathcal{X}}
\newcommand{\M}{\mathcal{M}}

\newcommand{\x}{\mathbf{x}}
\newcommand{\y}{\mathbf{y}}
\newcommand{\yp}{\mathbf{y'}}
\newcommand{\yt}{\tilde{\mathbf{y}}}
\newcommand{\xp}{\mathbf{x'}}
\newcommand{\xpp}{\mathbf{x''}}
\newcommand{\ypp}{\mathbf{y''}}
\newcommand{\Layer}{\mathcal{L}}
\newcommand{\hL}{\hat{\mathcal{L}}}

\DeclareMathOperator{\dis}{d}

\usepackage{algorithm}
\usepackage{algorithmic}
\usepackage{adjustbox}

\newtheorem{ques}{Question}
\newtheorem{Prop}{Proposition}
\newtheorem{assumption}{Assumption}
\newtheorem{definition}{Definition}
\newtheorem{theorem}{Theorem}
\newtheorem{lemma}{Lemma}
\newtheorem*{prop*}{Proposition}
\newtheorem*{theorem*}{Theorem}
\newtheorem*{lemma*}{Lemma}

\newtheorem*{obs*}{Observation}
\newtheorem*{Model*}{Model assumption}

\usepackage{enumitem,kantlipsum}

\usepackage{tikz}
\usetikzlibrary{calc}
\usepackage{etoolbox}

\usepackage{hyperref}

\usepackage[numbers,compress,sort]{natbib}	% for numerical citations
\title{CoreSPECT: Enhancing Clustering Algorithms via an Interplay of Density and Geometry}

\title{CoreSPECT: Enhancing Clustering Algorithms via an Interplay of Density and Geometry}

\author{
  Chandra Sekhar Mukherjee\footnotemark[1]~~\footnotemark[2]~,
  Joonyoung Bae\footnotemark[1]~~\footnotemark[3]~,
  Jiapeng Zhang\footnotemark[4]
  \ \\
  \\ 
  Thomas Lord Department of Computer Science 
  \\
  University of Southern California
}

\renewcommand{\undertitle}{}
\renewcommand{\shorttitle}{Enhancing clustering with CoreSPECT}

\begin{document}

\renewcommand{\thefootnote}{\fnsymbol{footnote}}

\maketitle

\footnotetext[1]{Equal contribution.}
\footnotetext[2]{chandrasekhar.mukherjee07@gmail.com}
\footnotetext[3]{joonyoungbae.aaron@gmail.com}
\footnotetext[4]{jiapengz@usc.edu}

% Restore normal numbering (optional)
\renewcommand{\thefootnote}{\arabic{footnote}}

%\maketitle

\begin{abstract}
In this paper, we provide a novel perspective on the underlying structure of real-world data with ground-truth clusters via characterization of an abundantly observed yet often overlooked \emph{density–geometry} correlation, that  manifests itself as a multi-layered manifold structure.

We leverage this correlation to design CoreSPECT (Core Space Projection based Enhancement of Clustering Techniques), a general framework that improves the performance of generic clustering algorithms. Our framework boosts the performance of clustering algorithms by applying them to strategically selected regions, then extending the partial partition to a complete partition for the dataset using a novel neighborhood graph based multi-layer propagation procedure. 

We provide initial theoretical support of the functionality of our framework under the assumption of our model, and then provide large-scale real-world experiments on 19 datasets that include standard image datasets as well as genomics datasets. 

We observe two notable improvements. First, CoreSPECT improves the NMI of K-Means by $20\%$ on average, making it competitive to (and in some cases surpassing) the state-of-the-art manifold-based clustering algorithms, while being orders of magnitude faster.

Secondly, our framework boosts the NMI of HDBSCAN by more than $100\%$ on average, making it competitive to the state-of-the-art in several cases \emph{without requiring the true number of clusters and hyper-parameter tuning}. The overall ARI improvements are higher.

\end{abstract}

\section{Introduction}
Density and geometry have long served as two of the fundamental guiding principles in clustering algorithm design, with algorithms usually focusing either on the density structure of the data (e.g., HDBSCAN~\cite{hdbscanLibrary} and Density Peak Clustering~\cite{DensityPeak,ADPClust,PECANN}) or the complexity of underlying geometry (e.g., manifold clustering algorithms). These paradigms each have some benefits and disadvantages. 

While simple-geometry-based algorithms (such as K-Means~\cite{LLyod-Kmeans}) can be extremely fast, they can have suboptimal performance if the clusters have non-spherical shape. On the other hand, manifold clustering algorithms enjoy stronger theoretical guarantees~\cite{annalittle,trillos} and empirical performance~\cite{von2007tutorial,shinde2024geometric} in the presence of complex geometry, but can be prohibitively slow for large datasets. 

On the other hand, while density based clustering algorithms can handle more nonlinearity than K-Means while being faster than manifold clustering algorithms and can work without needing the true number of clusters in some case, they often suffer from sensitivity to noise, as well as choice of hyperparameters.

Contrasting with the typical focus on either the density or geometry of data, we observe that in many datasets, the geometric notion of cluster separability is closely correlated with the underlying density. This insight motivates the following contributions.

\begin{figure}
    \centering
    \includegraphics[width=0.9\linewidth]{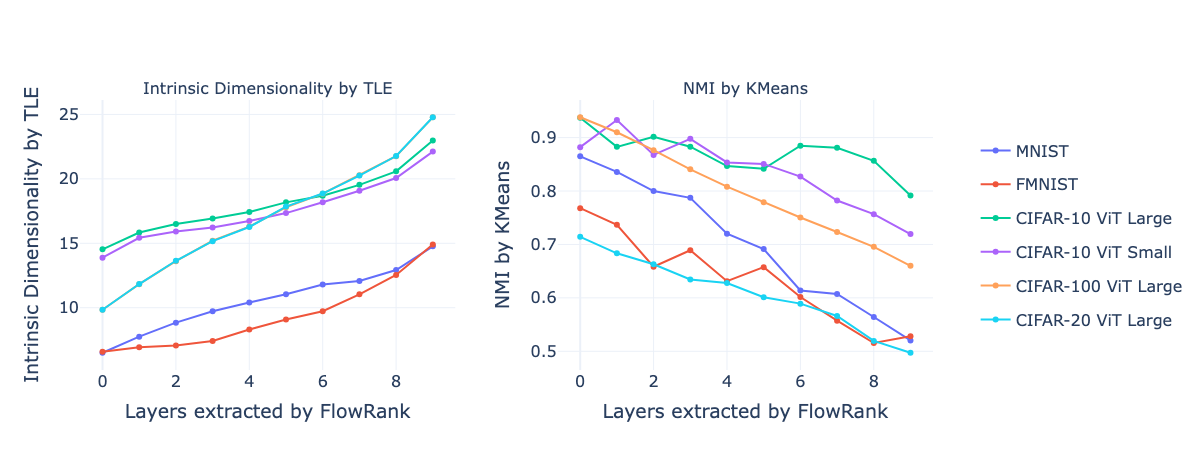}
    \caption{{\bf Increasing dimensionality} and {\bf Degrading K-Means performance} from inner to outer layers in image datasets. The layers are defined as deciles of points based on FlowRank score presented in Algorithm~\ref{alg: FlowRank}.}
    \label{fig:kmeans_evidence}
\end{figure}

\subsection{Contributions:}
\vspace{-0.5em}

\begin{enumerate}
    \item We observe a novel density-geometry correlation where the {\bf central, dense parts of the ground truth clusters are easily separable}, surrounded by lower-density regions that are both harder to separate and have more {\bf nonlinear manifold-like geometry}. 

    \item Using this observation we design a 4-step framework that is able to significantly boost the performance of clustering algorithms such as K-Means and HDBSCAN. 

    \item We test the performance of our algorithm on 19 large datasets, and observe that we make K-Means competitive with the state-of-the-art clustering algorithm, while being 50x faster, and also achieve very strong improvement in HDBSCAN in a \textbf{parameter-free} manner, both in performance as well as efficiency. %Notably, the performance of CS-HDBSCAN in the default settings significantly outperforms the tuned HDBSCAN performance both using the true number of clusters, and taking the best outcome across parameter search in a \emph{supervised} manner. We provide a result of this in Figure~\ref{}.
    
\end{enumerate}

To the best of our knowledge, an efficient clustering framework using aforementioned density-geometry correlations has not been studied in the literature before. Next, we proceed with a detailed description of our contributions.

\begin{enumerate}[wide, labelindent=3pt]
    \item {\bf A novel density-geometry correlation: density-driven multi-layered geometry} 
    
    We observe that in many datasets with ground truth clusters that are known to have non-linear, manifold like geometry, two complementary phenomena occur.

    \begin{itemize}[wide, labelindent=20pt]
        \item[a)] It is well known that standard image datasets like MNIST, Fashion-MNIST, and CIFAR-10/100 have a \emph{manifold geometry} with  \emph{low-intrinsic dimensionality}~\cite{manifold-hypothesis}. We observe that for \emph{all} of these datasets, the relatively denser region of the data have \emph{lower dimensionality compared to the whole data}, and as we consider the regions of lower density, the dimensionality goes up. We use the TLE algorithm~\cite{TLE_dim} that is known to capture intrinsic dimensionality under mild flatness conditions~\cite{dimensionality-survey2025} that have been also used in recent manifold clustering algorithm~\cite{shinde2024geometric}.

        \item[b)] The manifold-like geometry has motivated the design of several manifold-based clustering algorithms~\cite{von2007tutorial,trillos,annalittle,shinde2024geometric}. We observe that for \emph{all} of these datasets, the performance of K-Means on the relatively denser parts of the data (which we identify with our FlowRank algorithm) is significantly higher.
        On the other hand, as we include the sparser parts of the data, the performance of K-Means consistently degrades. 
    \end{itemize}

    These two phenomena are shown in Figure~\ref{fig:kmeans_evidence}. This presents an interplay between density, geometry, and separability that to the best of our knowledge has not been observed before. We obtain the same observation for the genomics dataset in the Supplementary Material. We call the dense regions core, and the sparser regions peripheries, and quantify this as the Layered Core Periphery based Density model,  $\sf{LCPDM}$ primarily focusing on the density-separability axis. The model is described in detail in Section~\ref{sec: model}, with more statistical evidence present in Appendix~\ref{app:stat}

\begin{figure}
    \centering
    \includegraphics[width=0.95\linewidth]{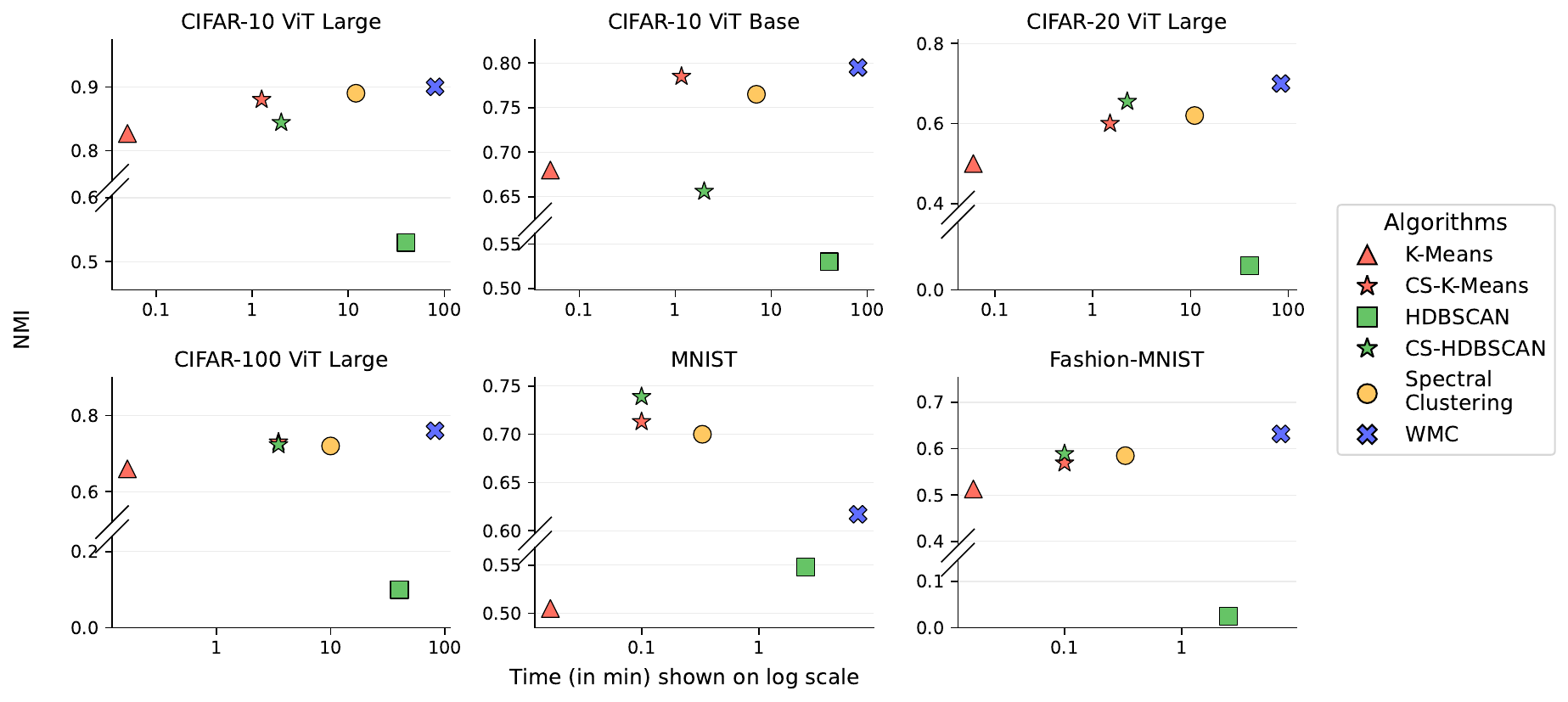}
    \caption{\small Performance vs. time plots of manifold clustering algorithms and CoreSPECT on image datasets.    
    We observe that both CS-K-Means and CS-HDBSCAN obtain similar performance to the highly efficient implementation of K-NN-based spectral clustering as well as the SOTA manifold clustering algorithm for CIFAR (WMC~\cite{shinde2024geometric}), while being significantly faster. Impressively, CoreSPECT boosts HDBSCAN to be competitive with the manifold clustering algorithms in several cases, without requiring knowledge of true number of clusters.
    }
    \label{fig:image-cifar}
\end{figure}

\begin{figure*}[t]
    \centering
    \includegraphics[width=0.8\linewidth]{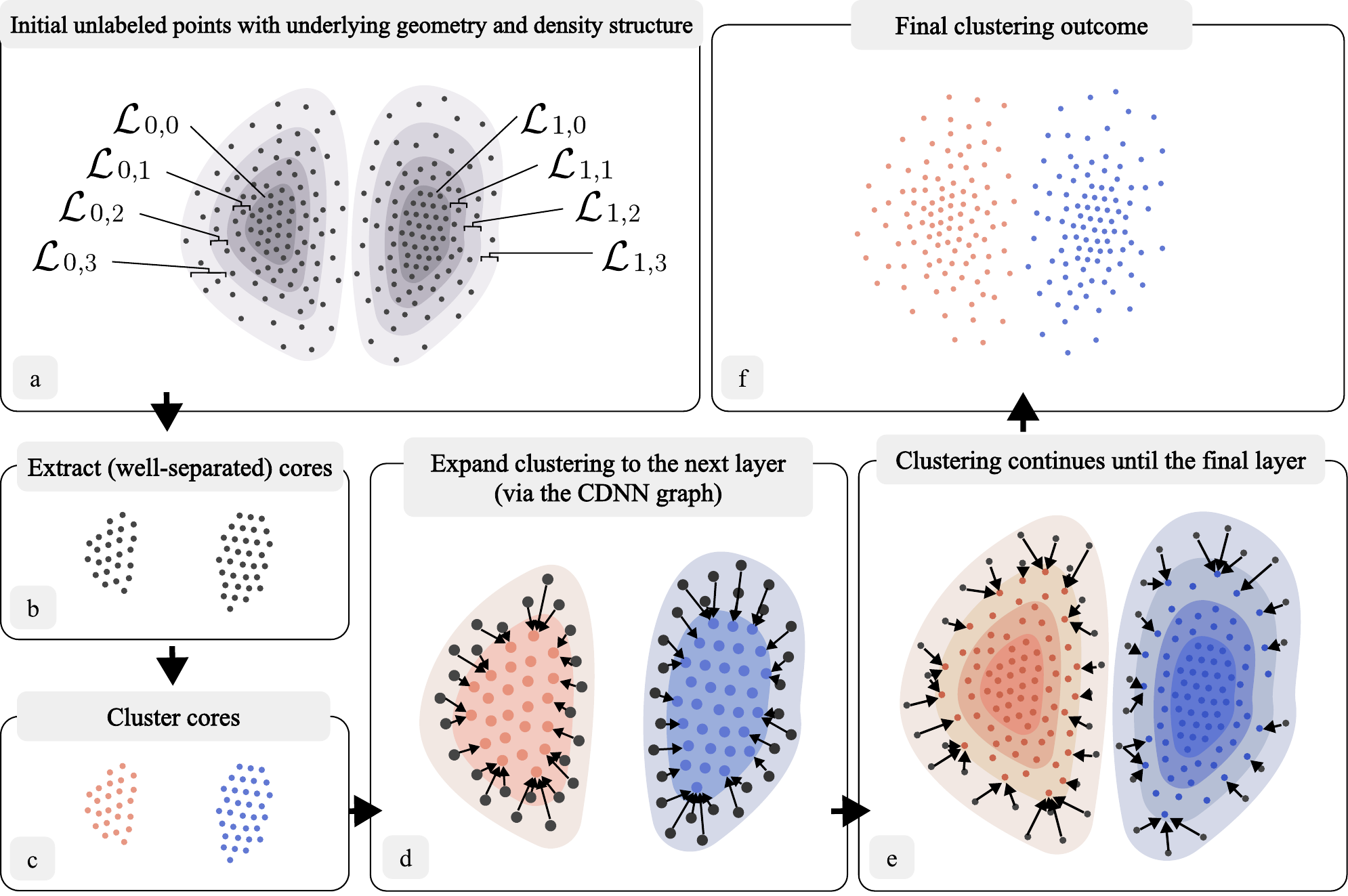}
    \caption{\small The underlying geometry-density structure in the data (a), and the step-by-step execution of the CoreSPECT framework. In step (b) we extract the cores. Then, in step (c) we cluster the cores with a simple algorithm. Steps (d) $\to$  (f) exhibit the layer-wise expansion of clustering using the CDNN graph.}
    \label{fig:expo}
\end{figure*}

    \item {\bf Generic clustering enhancement framework:} Using these insights, we design a generic \emph{$4$-step clustering enhancement framework (Figure~\ref{fig:CoreSPECT}}) targeted at datasets with underlying $\sf{LCPDM}$ model. In short,

    \subitem (a) We approximate the layers across the clusters with a notion of \emph{relative centrality} ~\cite{mcpc}, which we refine on to design a new algorithm FlowRank (Algorithm ~\ref{alg: FlowRank}).
    
    \subitem (b) We select the most central layers, which corresponds to the \emph{relatively highest density} regions of the data (which are easily separable as per our model). 

    \subitem (c) Cluster those parts of the data with a generic algorithm (such as K-Means or HDBSCAN).

    \subitem (d) Expand the clustering layer-by-layer using a novel layer-wise nearest-neighbor graph construction.  

    We provide a schematic run-through of the framework in Figure~\ref{fig:expo}.

    \item {\bf Experiments on large scale genomics and image datasets:}  

    Importantly, we observe that our observation applies to a large variety of important datasets. As evidence, we test our framework on {\it 19 datasets from three domains}  (single-cell RNA-seq, bulk-RNA seq, and popular image datasets), of {\it size ranging from $2,000$ to $50,000$ and dimension $50$ to $768$}, on both {\it geometry based algorithms such as K-Means} (that requires the true number of clusters) as well as {\it density based algorithms such as HDBSCAN} (that does not require the true number of clusters). We refer the algorithms as CoreSPECTED-K-Means (CS-K-Means) and CS-HDBSCAN.

    \begin{itemize}
        \item {\it Improving K-means:} Across the datasets, we improve the NMI of K-Means by $18\%$, with the median improvement being $25\%$. 

        \item {\it Improving HDBSCAN:}
        We improve the NMI of HDBSCAN with default parameters by {\bf over 100\%}, with the median improvement of {\bf 200\%}.

        \item {\it Improving run-time of HDBSCAN:} For large datasets ($\ge$ 50,000 points), our framework speeds up HDBSCAN by a factor of 5-10. 

        \item {\it Matching SOTA manifold clustering algorithms efficiently:} We make K-Means competitive to the performance of the SOTA manifold clustering algorithm ~\cite{shinde2024geometric} in NMI on the popular CIFAR datasets, while using default parameters. Our framework achieves this with 50x faster runtime.

    \end{itemize}

\end{enumerate}
We present the results on 6 image datasets in Figure~\ref{fig:image-cifar} and 8 genomics dataset in Figure~\ref{fig:finla-ari}. The rest of the results are present in the Appendix. In the rest of the paper, we expand on the  aforementioned contributions.

\begin{figure*}
\centering
\noindent\fbox{\begin{minipage}{\textwidth}
\begin{enumerate}
    \item {\bf Step 1: Core (and subsequent density layers) extraction (Figure~\ref{fig:expo} (a), (b)).}  We have described in our model that each cluster has some $\ell$ layers $\Layer_{i,j}, 0 \le j \le \ell-1$. We first want to obtain $\hL_{0}:= \cup_{i=1}^k \hat{\Layer}_{i,0}$, that is, approximately the union of cores of all the clusters. Crucially, we do this \emph{without} any clustering, but rather exploiting the density structure in the hierarchies. We describe this step in  Section~\ref{sec:framework-step-1}. We extend this to get an estimate of union of the outer density layers $\hL_{j},1 \le j <\ell$. \textbf{We use $S_j$ to denote our estimate of $\hL_j$}.

    \item {\bf Step 2: Clustering the core (Figure~\ref{fig:expo} (c)).} In this step, the cores ($S_0$) are clustered with a simple algorithm of user's choice. In our explanations and dry run, we use K-Means. We describe the rationale behind this in Section~\ref{sec:framework-step-2}.

    \item {\bf Step 3: Core directed nearest neighbor (CDNN) graph construction (Figure~\ref{fig:expo} (d)).} Next, we design a  novel layer-by-layer nearest neighbor graph structure using the estimated density layers $S_j$ and show that it captures the underlying cluster structure of the data better than Euclidean distance. This is described in Section~\ref{sec:framework-step-3}.

    \item {\bf Step 4: Layer-wise expansion of clustering (Figure~\ref{fig:expo} (d)$\to$(f)).} Finally, we use the clustering of the core obtained in Step 2 and the graph in Step 3 to cluster each of the layers in the hierarchy in time \emph{Linear in the number of points, clusters, and nearest neighbors to check}. This is described in Section~\ref{sec:framework-step-4}.

\end{enumerate}
\end{minipage}}
%\begin{center}
   \caption{ The CoreSPECT Framework (See Figure~\ref{fig:expo} for a schematic representation and Algorithm~\ref{alg:CS-Kmeans} for its application to K-Means)}
   \label{fig:CoreSPECT}
%\end{center} 
\end{figure*}

\section{Model formulation and the CoreSPECT framework}
\label{sec: model}

In this Section we define our model assumptions on a high level. We require certain technical assumptions to prove our initial theoretical guarantees, which are presented in detail in Section~\ref{app:theory}, along with all the proofs. 

We assume points of each ground-truth cluster $i$ is being generated from some $\X_i \subseteq \mathbb{R}^d$, where all the points lie on some $m$-dimensional smooth manifold $\M$. We make the following assumptions on the geometry-density interaction, motivated by our observations in Figure~\ref{fig:kmeans_evidence}.

{\bf Concentric subspaces} The fundamental assumption we make is that $\X_i$ can be expressed as a hierarchy of concentric subspaces $\X_{i,0} \subset \X_{i,1} \subset \cdots \X_{i,\ell-2} \subset \X_{i,\ell-1} =\X_i$ We call the difference between the $j$ and $j+1$-th subspace as the $j$-th layer $\Layer_{i,j}$, defined as
\[
\Layer_{i,0}:=\X_{i,0}, \quad\quad     \Layer_{i,j}:= \X_{i,j} \setminus \X_{i,j-1}, \ 1 \le j <\ell
\]
Here we assume that each $\Layer_{i,j}$ are smooth and have a minimum and maximum depth in any direction. We call the $\Layer_{i,0}$ as \emph{cores} and the outer layers as more peripheral. We assume each layer is connected and has finite volume in a well-defined measure. We make the following assumptions, that we collectively term as the Layered-core-periphery-density-model ($\sf LCPDM$). 

{\bf The Layered-core-periphery-density-model ${\sf LCPDM}(k,\ell)$:}

\begin{enumerate}[wide, labelindent=5pt]
    \item {\it Cores are dense and peripheries are sparse:} The data is generated by sampling $n_{i,j}$ points from each layer $\Layer_{i,j}$ uniformly at random, such that there exists a  constant $C>1$ satisfying 
    \[\frac{n_{i,j}}{Vol(\Layer_{i,j})} > C \cdot \frac{n_{i,j+1}}{Vol(\Layer_{i,j+1})}
    \] 
    That is, for each cluster, the inner-most layer (the core) is the densest, and the density of the outer layer monotonically go down. We present this schematically in Figure~\ref{fig:expo} (a). We denote the points generated from $\Layer_{i,j}$ as $\hL_{i,j}$.

    \item {\it Cores are well-separated.} The first assumption on the geometry dictates that the core-layers of the clusters are well separated in an Euclidean sense. For any clusters $i,i'$
    \[
    \exists \mu_{i,i'}<0.5 \;\; \text{s.t.}\;\;
\max_{\substack{\x \in \hL_{i,0} \\ \xp \in \hL_{i,0}}} \|\x-\xp\|
\;\le\;
\mu_{i,i'} \cdot
\min_{\substack{\x \in \hL_{i,0} \\ \xp \in \hL_{i',0}}} \|\x-\xp\|
    \]

    When there are only two clusters in the data, we define $\mu := \mu_{0,1}$.

    \item {\it Layer-wise clustering membership alignment.}
    There exists $\delta<1$ such that for any $\mathbf{x} \in \Layer_{i,j}$ and any $i' \ne i$,
    $
    \min_{\mathbf{x'} \in \Layer_{i,j-1}} \|\mathbf{x}-\mathbf{x'}\| \le \delta \cdot \min_{\mathbf{x''} \in \Layer_{i',j-1} } \|\mathbf{x}-\mathbf{x''}\|
    $. 
    
\end{enumerate}

That is, any point in a cluster is closer to the boundary of the inner-layer of the same cluster than points in inner-layer of another cluster. Additionally, we do not make any assumptions on the geometry and proximity of points in the outermost layers.

%Furthermore, we assume that once these points are sampled following the aforementioned assumptions, they may also be perturbed by zero mean $\epsilon$-norm Gaussian noise. This results in a structure where geometrically, the outer layers may have increasingly complex shape, of which we make no assumption, beyond the interaction of the points with neighboring layers. Now, we move towards algorithm design with the aforementioned assumptions in mind.

\subsection{The CoreSPECT framework}
\label{sec: core-spect}
We first give a brief outline of our framework in Figure~\ref{fig:CoreSPECT}.  

%We stress that the specific algorithms for each of the step in our framework can be further improved as long as they satisfy certain properties that we describe. For demonstration of the framework, we will use two hard-to-separate clusters (0,2) of the popular {\sf Fashion-MNIST} dataset for which K-Means has a poor performance. We make certain curvature assumptions on the inner layer (but not the outer-most layer) for proving our theorems that we explain in detail in the supplementary material. 

%\begin{obs*}[Fashion-MNIST$\{0,2\}$: K-Means]
%The K-Means classification accuracy on the clusters $0$ and $2$ is $\approx0.51$, where an accuracy of $0.5$ indicates a random partition.
%\end{obs*}

\subsubsection{\bf Step 1: Extracting the cores (and the subsequent layers)}
\label{sec:framework-step-1}

We first recover the density layers, estimating the core $\hL_0$ with $S_0$ that is well-separated in an Euclidean sense, and also the subsequent layers.

%Our definition dictates that the densest parts of the communities are the most well-separated. Then, if we are able to (approximately) identify the cores , i.e., $\cup_{i=1}^k \X_{i,0}$, these parts would be well clusterable. We aim to do something more extensive. We aim to get a data points of the vertices with the following properties. 

\begin{definition}[Layer-preserving ranking]
\label{def: layer-preserving}
Given a dataset $X$ generated from the ${\sf LCPDM}(k,\ell)$ model, we say a ranking of the points in $X$ is layer-preserving if for any cluster $\X_{i}$, 
if $\x \in \hat{\Layer}_{i,j}$ and $\xp \in \hat{\Layer}_{i,j'}$ where $j<j'$, then $\x$ is ranked above $x'$. 
\end{definition}

Our ranking algorithm is motivated by the \emph{relative centrality framework} of \cite{mcpc}, designed to obtain the central parts of \emph{each community}. In contrast, we aim to find the \emph{densest regions of each cluster}. In this direction, we obtain a q-NN graph $G_{q,X}$ and obtain the distribution of $\log (n)$ step random walk, denoted as $\Pi(G_{q,X})$, mimicking \emph{initial centrality} of \cite{mcpc}. However, we note (and explore in the appendix) that their notion of relative centrality does not capture the density hierarchies in the underlying data. To circumvent this, we define the concept of \emph{relative density}.

{\bf Relative density estimation.} \quad Given a dataset $X$ with density estimation $\Pi: V \rightarrow [0,1]$ for each data point, we look at its $r$-nearest neighbors, and randomly move to one of the neighbors that have a higher $\Pi$  value, continuing the process until we reach a maxima, as described in Algorithm~\ref{alg:RAWR}.

Then, our core-ranking algorithm,  \emph{FlowRank} (Algorithm~\ref{alg: FlowRank}) calculates the average between the density of a point, and that of the average maxima reached by randomly ascending random walks starting from that point. We provide the following guarantee for the ranking, with the proof (and the proofs of the following theorems) present in the Appendix.

\vspace{1em}
\begin{minipage}{0.48\linewidth}%[width=0.5\linewidth]
\begin{algorithm}[H]
   \caption{${\sf RARW}(X,\Pi,r,i)$ \\(Randomly Ascending Random Walk)}
   
   \label{alg:RAWR}
\begin{algorithmic}
   \STATE {\bfseries Input:} $X$, $\Pi: X \rightarrow [0,1]$ and index $i$.

    \STATE Let $N_r(\mathbf{x})$ be the $r$ closest points to $\mathbf{x}$.

    \STATE Start: $\mathbf{x}^+ \gets \mathbf{x_i}$
    \WHILE{ $\Pi(\mathbf{x}^+) < \underset{\mathbf{x'} \in N_r(\mathbf{x}^+)}{\max} \Pi(\mathbf{x'})$}
    \STATE $N^+ \gets \mathbf{x'} \in N_r(\mathbf{x}^+) : \Pi(\mathbf{x'})>\Pi(\mathbf{x}^+)$ 
    \STATE Randomly select a point $\mathbf{x''}$ from $N^+$
    \STATE $\mathbf{x}^+ \gets \mathbf{x''}$
    
    \ENDWHILE

    \RETURN $\Pi(\mathbf{x}^+)$
\end{algorithmic}
\end{algorithm}
\end{minipage}
\begin{minipage}{0.48\linewidth}
\vspace{-0.8em}
\begin{algorithm}[H]
   \caption{FlowRank: ${\sf FR}(X,\Pi,q,r)$}
   \label{alg: FlowRank}
\begin{algorithmic}
   \STATE {\bfseries Input:} $X$, a $n \times d$ dataset, neighborhood parameter $r$,
   density vector $\Pi$. 

    \STATE

    \STATE (We obtain the density estimation through a $\log n $-step random walk simulation on $G_{q,X}$.)

    \STATE
    
    \FOR{i in 1:n}

    \STATE $z_i \gets  \mathbb{E} [{\sf RARW}(X,\Pi,r,i)] $ \COMMENT{Algorithm~\ref{alg:RAWR}}

    \STATE $score[i] \gets  \frac{\Pi_i}{z_i}$

    \ENDFOR
    
\end{algorithmic}
\end{algorithm}
\end{minipage}

%Then, our cor`f{alg: FlowRank} is defined as follows. We start by obtaining an estimate of the density from the random walk distribution. Then, from each node $i$, we obtain an estimate of the expected score achieved by running Algorithm~\ref{alg:RAWR} several times. 

\begin{theorem}[Core-detection by FlowRank]
\label{thm: FlowRank}
Let data be generated from the ${\sf LCPDM}(2,\ell)$ model. Let $\Pi$ be the density of the space. Then, for some $r=\mathcal{O}(1)$, on expectation, all the core points ($\hL_0$) get a score of $1$. Additionally, all non-core points get a score $<1$.
\end{theorem}

The proof can be found in Appendix ~\ref{sec:proof-FlowRank}. Here we note that we make certain extra assumptions on the curvature of the manifold for theoretical completeness, which we note down in Appendix~\ref{app:detailed-model}.

In practice, 
%we do not have an exact density estimation, and real-data will not follow the exact structure proposed by us. Nonetheless, 
we observe selecting top $10\%$ ranked points as the core works well, as it consists of points from each underlying cluster and we will observe the top-points are also very separable. Additionally, we define the next 10\% blocks as our proxy for the layers $\hL_j$. Here we note that in different datasets, the number of layers may vary, however, if the ranking is layer preserving, our partition either breaks a single layer into multiple ones, or only merges multiple \emph{consecutive} layers into one. Obtaining a better approximation of $\hL_j$ is an interesting direction towards further improving our framework.

%This is motivated by the layer-preserving structure of Definition~\ref{def: layer-preserving}. W, which we shall use in the next step.

%\begin{obs*}[Fashion-MNIST$\{0,2\}$: Core extraction]
%The top $10\%$ ranked points obtained by FlowRank contain an almost equal number of points from each cluster.
%\end{obs*}

\subsubsection{Step 2: Applying the clustering algorithm to the core}
\label{sec:framework-step-2}

Here, we formally (as well as experimentally) capture the separability of the cores by K-Means.

\begin{Prop}
\label{prop:k-means-separate}
Let $n_0$ and $n_1$ points ($n_0>n_1$ WLOG) be sampled from  $\X_{0,0}$ and $\X_{1,0}$ respectively such that $\mu \cdot n_0 \le n_1$. These points are denoted $\hL_{0,0}$ and $\hL_{1,0}$. Consider the slightly modified K-Means algorithm. Obtain two centers using the K-Means++ method and run one-step K-Means. Repeat this process some $2 \log n$ times and accept the result with minimum K-Means objective value. This algorithm separates the two cores correctly with probability $1-o(1)$.
\end{Prop}

The proof can be found at the restated proposition 1 in Appendix ~\ref{sec:prop_one}
%\begin{obs*}[Fashion-MNIST$\{0,2\}$: K-Means on core]
%The clustering accuracy of K-Means on these points is $\mathbf{0.87}$, up from an initial $\mathbf{0.51}$.
%\end{obs*}

\begin{figure*}
\includegraphics[width=1.01\textwidth]{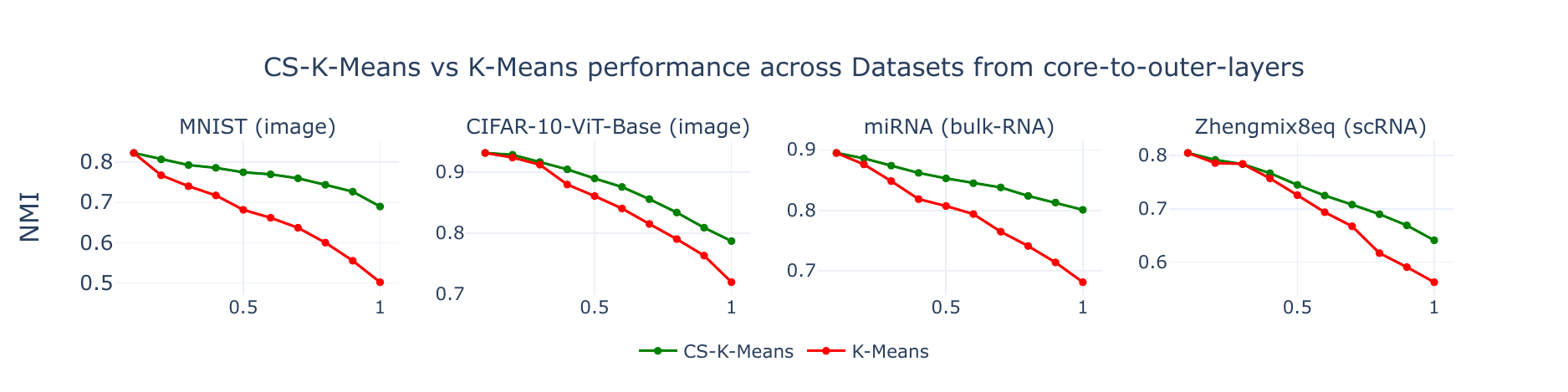}
  \caption{Comparing the accuracy of applying K-Means on the top $x$-fraction of the points (according to FlowRank) vs. applying K-Means to top $10\%$ and then applying Layer-wise expansion (Algorithm~\ref{alg:prop}) up to top $x$-fraction, for $x \in [0.1, 0.2, \ldots, 1]$ for some pairs of clusters. For $x=0.1$ (the cores), K-Means has very good performance, but as we apply K-Means on the outer layers, its performance deteriorates. In contrast, the layer-wise propagation leads to significantly lower decay, leading to improvement in the overall performance.} 
  \label{fig:side-by-side}
\end{figure*}

\begin{comment}
\begin{figure*}[t]
  \centering
    \subfigure[Fashion-MNIST (image)]{
    \includegraphics[width=0.305\textwidth]{layer-wiseFashion-MNIST-0-2.pdf}
    \label{fig:layered-fm}
    }
    %\caption{}
  %\end{subfigure}
  \subfigure[mRNA (TCGA) (bulkRNA)]{
    \includegraphics[width=0.305\textwidth]{layer-wisemRNA-5-20.pdf}
    \label{fig:layered-mrna}
  }
  \subfigure[Zhengmix8eq (scRNA)]{
    \includegraphics[width=0.305\textwidth]{layer-wiseZhengmix8eq-1-4.pdf}
    \label{fig:layered-zh}}
  %\hfill
  %\hfill

  \caption{Comparing the accuracy of applying K-Means on the top $x$-fraction of the points (according to FlowRank) vs. applying K-Means to top $10\%$ and then applying Layer-wise expansion (Algorithm~\ref{alg:prop}) up to top $x$-fraction, for $x \in [0.1, 0.2, \ldots, 1]$ for some pairs of clusters. For $x=0.1$ (the cores), K-Means has very good performance, but as we apply K-Means on the outer layers, its performance deteriorates. In contrast, the layer-wise propagation gives significantly better response.} 
  \label{fig:side-by-side}
\end{figure*}
    
\end{comment}

\begin{algorithm}[ht]
   \caption{Layer-wise-expansion of clustering: {\sf Expansion}($S,G^+,W,\mathcal{C}$)}
   \label{alg:prop}
\begin{algorithmic}
   \STATE {\bfseries Inputs:} 
    The CDNN graph $G^+_{t,S}$ (with weight matrix $W$), layers $S$, cluster membership vector $\mathcal{C}: S_0 \rightarrow \mathbb{R}^k$.
   
    \STATE Initiate clustering: $V_1,\ldots ,V_k$ : $\forall \mathbf{x} \in S_0, \mathbf{x} \in V_{t}$
    
    \STATE \hspace*{10em} $\text{where } t:= \arg\min_{i} \mathcal{C}[\mathbf{x}]_i$.

    \STATE Define data structure $\hat{C}$:  $\forall \mathbf{x} \in S_0, \hat{C}(\mathbf{x}) \gets \mathcal{C}(\mathbf{x})$
  
    \FOR {$j \in 1:\ell$}

    \FOR {$\mathbf{u} \in S_{j}$}

    \STATE $\hat{C}(\mathbf{u}) \gets \underset{v \in N_{G^+}}{\sum} W(\mathbf{u},\mathbf{v}) \cdot \hat{C}(\mathbf{v}) $;
     \\ $k_u \gets {\sf arg}\min \hat{C}(\mathbf{u})$;
     \quad $V_{k_u} \gets V_{k_u} \cup \mathbf{u}$
    
    \ENDFOR

    \ENDFOR 

    \RETURN Clustering $V_1, \ldots ,V_{k}$
   
\end{algorithmic}
\end{algorithm}

\paragraph{Non-linearity in the outer layers:}
While the cores are separable by K-Means, we  observe that as we include the outer layers, the performance of K-Means degrades. This can be attributed to the fact that the shape of the data becomes more non-linear as we move away from the center. We observe that as we move away from the center, shortest-path-distance on a nearest-neighbor embedding becomes an increasingly better estimate of cluster membership compared to the Euclidean distance. This is shown in Appendix~\ref{sec: stat-1}. However, shortest-path-based approaches can be computationally expensive even with approximations~\cite{annalittle,trillos}.

Instead, we exploit the density structure and build a graph that allows us to \emph{efficiently} cluster the rest of the points building on our third model assumption that cluster membership of points in outer layers are better captured by nearby points in inner layers.

\begin{algorithm}[ht]
   \caption{CoreSPECTed-K-means}
   \label{alg:CS-Kmeans}
\begin{algorithmic}
    \STATE {\bfseries Inputs:} Dataset $X$ and number of clusters $k$, hyperparameters: $r,\ell,t$.

    \STATE 

    \STATE {\bf Step 1:} Obtain $\Pi: X \rightarrow [0,1]$ a random walk based estimate of the density of each point.
    
    \STATE  $F \gets {\sf FlowRank}(X,\Pi,q, r)$ \hfill \COMMENT{Algorithm~\ref{alg: FlowRank}}

    \STATE  $\forall j \in \{0,\ldots, \ell-1\}$, define $S_j$ as the top $(j+1)/\ell$ fraction of points as per $F$.   

    \STATE
    \STATE {\bf Step 2:} Cluster $S_0$ with K-Means, obtaining centroids $\mathbf{c_i}$. 

    \STATE Define cluster membership vector $\mathcal{C}: S_0 \rightarrow \mathbb{R}^k$: $\mathcal{C}(\mathbf{u})=[\|\mathbf{u}-\mathbf{c_i}\|]_{1 \le i \le k}$
    
    \STATE 

    \STATE {\bf Step 3}: Generate the CDNN graph $G^+_{t,S,X}$  and normalized weight matrix $W$\hfill \COMMENT{As per Definition~\ref{def: CDNN}}

    \STATE 
        
    \STATE {\bf Step 4}: $V \gets {\sf Expansion}(S,G^+,W,\mathcal{C})$ \hfill \COMMENT{Algorithm~\ref{alg:prop}}

    \STATE 

    \RETURN Clustering $V_1, \ldots ,V_{k}$
   
\end{algorithmic}
\end{algorithm}

\subsubsection{Step 3: Creating the Core Directed Nearest Neighbor (CDNN) graph}
\label{sec:framework-step-3}

Our model assumption dictates any point in $\hat{\Layer}_{i,j}$ will be close to the points of the neighboring inner layer of the same ground truth cluster, irrespective of other distances. This motivates the following graph embedding.

\begin{definition}[Centrally directed nearest neighbor graph (CDNN) $G^{+}_{(t,S)}$]
\label{def: CDNN}
For every datapoint in $S_j$ (the layers obtained in the previous step), we connect it to some $t$ nearest neighbors in $\cup_{j'=0}^{j-1} S_{j'}$.\footnote{Ideally, we want a graph that connects vertices in $S_{j}$ to vertices in $S_{j-1}$, but we choose the aforementioned formulation to limit propagation of error in the FlowRank outcome.}
\end{definition}

We observe that distance on the CDNN graph is also more reliable than Euclidean distances. However, due to the layer-wise structure of the CDNN graph, we are able to devise a very fast algorithm to cluster the rest of the points.

\subsubsection{Step 4: Expanding the clustering in a layer-wise manner}
\label{sec:framework-step-4}

We instead expand the clustering a layer at a time. That is, given a clustering of the points in $S_0$, we label the points in $S_1$ using the edges going from $S_1$ to $S_0$, continuing this process layer-wise.

For this, we need two ingredients. First, for each point in $S_0$, we define a $k$-dimensional vector $\mathcal{C}$(each entry corresponding to a cluster) that gives us the membership value of the point to the cluster according to the clustering algorithm we apply to the core. For example, it can be the distance to the $k$ centroids for K-Means.

%First, we obtain a clustering of the top layer $S_0$ into some $k$ communities. Consequently, for each vertex $v_0$ we have a notion of the confidence of it belonging to any cluster. For K-Means, it is the distance to the centroid, where closer the distance, the more confidently one can assume that point belongs to the cluster determined by K-Means (low value is better). Similarly for GMM, each point gets a $k$-dimensional vector that determines its likelihood of belonging to each cluster (high value is better). Depending on the notion of confidence in the $k$-dimensional vectors, our algorithm will seek to minimize (or alternatively maximize) a certain object.  

Next, we define a weight function for each edge in $G^+_{t,S}$ that is inversely proportional to the distance. For this paper we use UMAP's weight function~\cite{mcinnes2018umap}, and observe the behavior w.r.t. different choices in the supplementary material. We use these two vectors to cluster the subsequent layers one at a time, which we describe in algorithm~\ref{alg:prop}. The performance of this propagation step is captured in Figure~\ref{fig:side-by-side}. Finally, we record the computational efficiency of this step. 

\begin{theorem}
\label{prop: time}
Given the CDNN graph $G^+_{t,S}$ and a clustering of $S_0$, the rest of the points can be clustered in $\mathcal{O}(n \cdot k \cdot t)$ time, which is linear in the number of the edges and the number of clusters in $G^+_{t,S}$. 
\end{theorem}

The proof can be found at the restated theorem 2 in Appendix ~\ref{sec:theo_two}

%\begin{minipage}{0.5\linewidth}

%\end{minipage}
% \begin{minipage}{0.5\linewidth}
% \begin{figure}[H]
%     \centering
%     \includegraphics[width=\linewidth]{images/layer-wise.pdf}
%     \caption{We compare Layer-wise expansion applied to K-Means with K-Means applied to the inner-most layers. As the number of layers increase, the accuracy due to K-Means continue to go down. In comparison, layer-wise expansion ensures a significantly higher quality clustering.}
%     \label{fig:enter-label}
% \end{figure}
% \end{minipage}

\begin{comment}
\begin{minipage}{0.6\linewidth}
\begin{algorithm}[H]
   \caption{Single-layer-propagation}
   \label{alg:prop}
\begin{algorithmic}
   \STATE {\bfseries Input:} Dataset $X$, two layer hierarchy $S_0 \subset S_1=X$ 
   \STATE and K-Means clustering assignment on $S_0$.
   \STATE 
    \STATE   Build $\hat{G} \gets G^{Q+}_{(X,S)}$    
    
    \FOR{ $v$ in $S_1 \setminus S_0$}

    \FOR {u in $N_{G^{Q+}}(v)$}

    \STATE $W[v,u] \gets e^{\|u-v\|/\epsilon}$
    
    \ENDFOR

    \STATE $W[v,:] \gets \frac{W[v,:]}{\sum_{u \in N_{G^{Q+}}(v)}}$

    \ENDFOR
    \STATE $v^+ \gets v$
    \WHILE{ $P(v) < \max_{u \in N_G(v) P(u)}$}
    \STATE $N_G^{+}(v) \gets u \in N_G(v) : P(u)>P(v)$ 
    \STATE Randomly select an item $v^+$ from $N_G^{+}(v)$
    \STATE $v \gets v^+$
    
    \ENDWHILE

    \RETURN $P(v^+)$
\end{algorithmic}
\end{algorithm}
\end{minipage}
\end{comment}

\begin{figure*}
    \centering
    \includegraphics[width=0.95\linewidth]{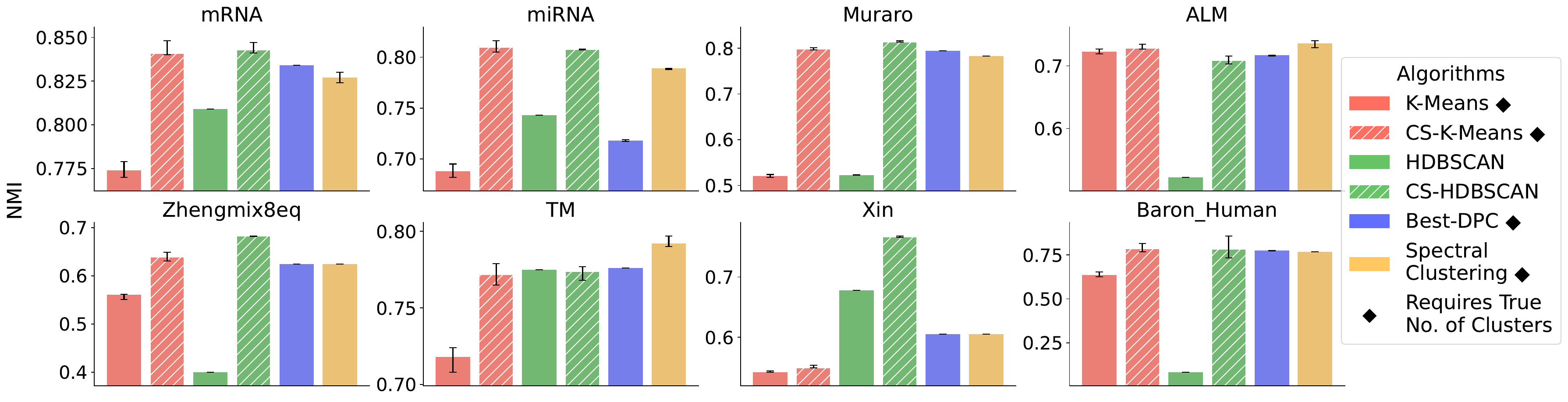}
    \caption{\small Improvement of NMI on K-Means and HDBSCAN due to CoreSPECT, compared to best density-peak-based clustering as well as spectral clustering. Impressively, CS-HDBSCAN performs on par (sometimes even being the best) compared to popular algorithms that need the true number of clusters.}
    \label{fig:finla-ari}
\end{figure*}

\subsection{Combining the framework explicitly for K-Means}

Here, we write down the different steps of our framework when applied to K-Means as Algorithm~\ref{alg:CS-Kmeans}.

%\begin{minipage}{0.6\linewidth}
%\end{minipage}

\paragraph{Theoretical guarantee of CoreSPECTed-K-Means.}

Finally, we note that as long as each of our steps is approximately correct, CoreSPECTed-K-Means almost correctly recovers the underlying clusters in data that follows our model, irrespective of the proximity of points in the outer layers. 

\begin{theorem}[Clustering in the $\sf LCPDM$ model]
\label{thm:main-sep}
Let $X$ be $n$ datapoints generated from the ${\sf LCPDM}(2,\ell)$ model. Let us have an estimate of the density layers, given as $S_0, \ldots ,S_{\ell-1}$ such that $|S_j \cap \hL_j|=(1-f)|\hL_j|,j>0$. for a sufficiently small function $f=o(1)$ that is layer preserving. Then, applying a variant of K-Means to $S_0$ and expanding the clustering using a CDNN graph (with correctly chosen parameters) results in clustering with o(1) misclassification error rate on expectation.
\end{theorem}

The proof can be found at the restated theorem 3 in Appendix ~\ref{sec:theo_three}

%While we obtain theoretical justification for our framework, obtaining these results under more generic conditions is an important next step. Additionally, studying the behavior of our framework in settings like Gaussian distribution supported on special kinds of manifolds is an interesting problem.

\vspace{-0.5em}
\paragraph{Computational efficiency}
We note that our framework is efficient. The main computation involves generating nearest neighbor graphs (both for initial density estimation as well as the CDNN graph generation), for which we used the ultra-fast HNSW~\cite{hnswlib} library for approximating the nearest neighbors. Our framework is also linearly dependent on $k$, the target number of clusters.  For example, for CoreSPECTed-K-Means, our framework terminates in a few seconds for most of the datasets, taking around a minute for the TM dataset, which is a $\approx 50$-dimensional dataset consisting of $\approx 50,000$ points and $54$ underlying clusters. We provide a detailed study of the run time (along with asymptotic runtime) of individual steps and scope for improvement in the Appendix ~\ref{sec:runtime_analysis}.

\vspace{-0.5em}
\paragraph{Ablation studies}
Here we have provided a run-through as well as initial theoretical support of our framework w.r.t. K-Means. In Appendix~\ref{app:ablation}, we provide several ablation studies to further test and support several of our algorithmic design choices. Specifically,

i) We test the relative usefulness of FlowRank compared to that of relative centrality coined in \cite{mcpc}. We observe that FlowRank has noticeably better performance.

ii) We also test the utility of our CDNN based clustering expansion compared to that of propagation methods popular in semi-supervised-learning, as it is a natural candidate. Our experiments showed that in many datasets the CDNN-based propagation method outperforms popular propagation methods. 

iii) Finally, we also test the usefulness of the weighting function in our propagation mechanism.

%\begin{comment}
\section{Large scale experiments}
\label{sec:exp}
In this Section we compile our real-world experimental results in detail. We focus on 19 datasets. This includes the 6 image datasets (MNIST, FMNIST, CIFAR10-ViT-Large, CIFAR10-ViT-Base, CIFAR20-ViT-Large and CIFAR100-ViT-Large), 11 single-cell RNA-sequencing datasets, and 2 bulk-RNA sequencing datasets. Appendix~\ref{app:exp} contains a detailed description of the datasets.

\subsection{Improvement in K-Means and HDBSCAN due to CoreSPECT}

\begin{obs*}[Improvement due to CoreSPECT]
Overall, {\it CoreSPECT improves the ARI of K-Means on 18/19 datasets}, and {\it NMI for 19/19 datasets} with all hyperparameters fixed to default. On average, {\bf ARI improves by $\mathbf{40.82\%}$} and {\bf NMI improves by $\mathbf{18.51\%}$}. 

For CS-HDBSCAN the performance improvement is even more substantial. HDBSCAN is known to be sensitive to noise, and as such has poor performance for most of the datasets considered here. However, the performance of HDBSCAN is significantly higher in the core, which again justifies our modeling. We observe an {\bf average ARI improvement of 468\%} and an {\bf average NMI improvement of 132\%}. 
Surprisingly, in the genomics dataset CS-HDBSCAN achieves the best NMI ranking, and is also close to the SOTA manifold clustering algorithm on image datasets, both being significantly faster, and not requiring true number of clusters. 
\end{obs*}

%Here, we note that the improvement due to CoreSPECT is lower for the $4$ smallest datasets (with less than 2500 points). In Appendix~\ref{}, we show that the performance of our framework can be further boosted with principled parameter selection methods including these small datasets.

\subsection{Comparison with other algorithms}
\label{sec:exp-comp}

We have observed that our framework consistently improves the performance of K-Means and HDBSCAN. Next, we want to compare the relevance of the improved performance w.r.t. more complex density-based and geometry-based clustering algorithms. 

For the image datasets, we focused on the k-nn-based spectral clustering, and the recent manifold-based-clustering~\cite{shinde2024geometric} that shown to be near-optimal to even deep-clustering algorithms in image datasets.

For the genomics datasets, we select three recent or/and popular Density Peak Clustering algorithms the original DPC~\cite{DensityPeak}, ADPclust~\cite{ADPClust}, and a fast implementation of another popular DPC algorithm through a fast framework named PECANN~\cite{PECANN}. Additionally we also use HDBSCAN~\cite{hdbscan}. We also attempted to use DBSCAN~\cite{dbscan} and OPTICS~\cite{ankerst1999optics}, but they categorized most points as outliers for most datasets so we omit them. We show the results in Figure~\ref{fig:finla-ari}. We also use the EM-based GMM-fit algorithm ~\cite{gmm-fit}. We also obtain a ranking of the NMI and ARI across the genomics datasets. The results are shown in Table~\ref{tab:rankings1}. Here we again emphasize that our results are obtained by running CoreSPECT with the \emph{same hyperparameters across all datasets}.

\subsection{Parameter selection process for CoreSPECT}
Our framework uses some hyperparameters. Specifically, we have 3 neighborhood selection parameters for different stages $q,r,t$ and thresholds for defining the layers. To ensure a fair comparison, we tested some common neighborhood values on two datasets MNIST and miRNA, and found that $q=40$, $r=t=20$ works well here. Then, we applied the \emph{same hyperparameters} on all the other 17 datasets. The hyperparameters for each of these datasets could be further tuned to obtain better results, but we use these default hyperparameters throughout to design an unsupervised learning algorithm that can be generally applicable.

%We also attempted to run other manifold-based algorithms with theoretical guarantees by \cite{annalittle} and \cite{trillos}. However, we could not get a parameter initialization (after trying default and other suggested parameters) that gave better ARI or NMI than the vanilla Spectral Clustering (SC) on any of the datasets. Note that SC applies K-Means to the normalized spectral embedding of the nearest neighbor graph~\cite{von2007tutorial}. We also use bisecting K-Means~\cite{bisecting-k-means} as a hierarchical clustering algorithm. 

\begin{table*}[ht]
  \centering
  \caption{\small NMI and ARI rank of CoreSPECTed-K-Means and HDBSCAN on 13 genomics datasets. 
  The best value is in bold, and the second best is underlined.}
  \label{tab:rankings1}

  \resizebox{0.95\textwidth}{!}{%
  \begin{tabular}{lcccccccccc}
    \toprule
      & K-Means & CS-K-Means & HDBSCAN & CS-HDBSCAN & GMM 
      & PECANN & SC & Bi-K-Means & DP & ADPClust \\
    \midrule
      ARI 
      & 5.80 & \textbf{2.73} & 6.87 & \underline{3.00} & 3.93
      & 3.47 & 4.20 & 7.07 & 8.67 & 9.00 \\

      NMI
      & 6.00 & \underline{2.80} & 7.27 & \textbf{2.60} & 4.40
      & 3.80 & \textbf{2.60} & 7.73 & 8.53 & 8.93 \\
    \bottomrule
  \end{tabular}
  }
\end{table*}

\section{Conclusion}

We conclude this paper with a small discussion of limitations and future directions. The most fundamental question lies in Figure~\ref{fig:side-by-side}. Our framework improves the performance of K-means (and HDBSCAN) by propagating the labels in a better way compared to simple K-Means. The limit to which we can further improve this propagation, both in a theoretical and in an applied setting is an outstanding problem in our opinion. Beyond this, we aim to make our framework  more efficient with parallelization based implementation. We discuss more limitations in Appendix~\ref{app:limitations}.

% \begin{table}[ht]
%   \centering
%   \resizebox{1\columnwidth}{!}{%
%     \begin{tabular}{lcccccccccc}
%       \toprule
%             & K‑Means & CS‑K‑Means & HDBSCAN & CS-HDBSCAN & GMM  & PECANN & SC    &  Bi‑K‑Means & DP    & ADPClust \\
%             \midrule
%         ARI & 5.80 & \textbf{2.73} & 6.87 & \underline{3.00} & 3.93 & 3.47 & 4.20 & 7.07 & 8.67 & 9.00    \\
%       NMI &6.00 & \underline{2.80} & 7.27 & \textbf{2.60} & 4.40 & 3.80 & \textbf{2.60} & 7.73 & 8.53 & 8.93       \\
 
%       \bottomrule
%     \end{tabular}%
%   }
% \caption{\small NMI and ARI rank of CoreSPECTed-K-Means on 15 datasets. The best value is in bold, and the second best is underlined}
% \label{tab:rankings1}
% \end{table}

%In Figure~\ref{fig:finla-ari}, we have shown the exact improvements in ARI due to CoreSPECTed-K-Means and CoreSPECT-GMM for 10 datasets. Here, we compare the performance over 15 datasets by demonstrating the average ranks for ARI (and NMI). As before, if we slightly tune the parameters for the small datasets, our performance further improves but in the interest of exhibiting robustness we use the same parameters for all datasets. We observe the rank of CoreSPECTed-K-Means compared to the other algorithms in Table~\ref{tab:rankings1}. Similarly, we note the rank the CoreSPECTed-GMM also achieves the best ARI rank, and the second best NMI rank over all of the datasets. Furthermore, with slight parameter tuning, CoreSPECTed-K-Means is able to achieve also the best rank in NMI. We describe all of these results in the supplementary material.

%\end{comment}

\bibliographystyle{alpha}
\bibliography{reference_cs}

\appendix

%\appendix

\section{Statistical Evidence of the LCPDM model in real-world datasets}
\label{app:stat}

So far, we have provided initial theoretical support as well as various ablation study to better demonstrate the structural backgrounds and performance of our framework. Here, we provide some more statistical evidence that supports our model formulation and algorithm design. First we describe the datasets we use for evaluation.

\paragraph{\large Datasets}
We use a total of $15$ datasets. We select the $11$ single-cell datasets used by \cite{mcpc} (obtained from \cite{single-cell-duo,single-cell-7data,single-cell-ALM-VISP}), two popular bulk-RNA dataset from the The Cancer Genome Atlas Program (TCGA)~\cite{TCGA}, and the popular image datasets MNIST~\cite{mnist-dataset} and Fashion-MNIST~\cite{fashion-mnist}. The details of the datasets are provided in the supplementary material.

First, we show that the inner-most layers are indeed more Euclidean  than the outer-most layer in terms of cluster-membership identities.

\begin{figure}[t]
    \centering
    \includegraphics[width=1\linewidth]{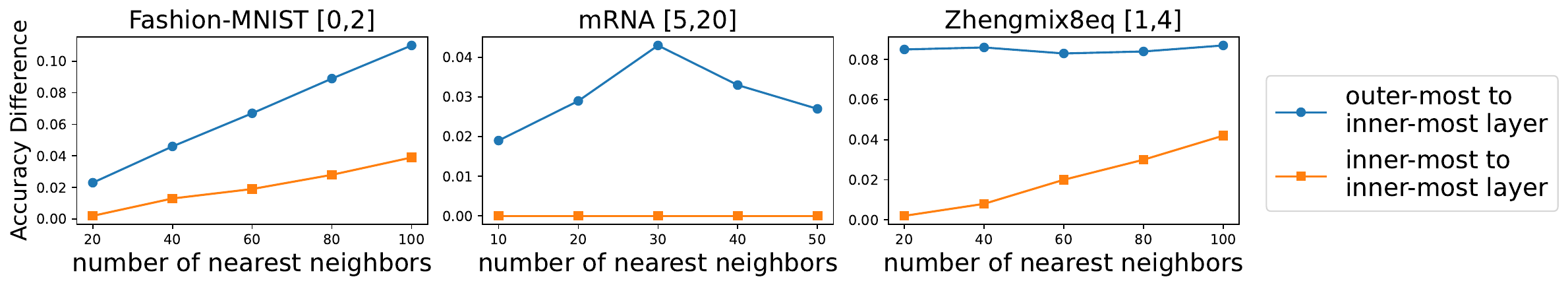}
    \caption{We first sort the nodes by the FlowRank values and select the nodes that are in the top and bottom 20\% which we call the nodes from the inner-most and outer-most layers.Then we find the nearest neighbors using Euclidean distance and the shortest path distance in the CDNN graph generated. We show an empirical evidence that the periphery nodes exhibit more non-linear and the core nodes exihibit more euclidean structure by showing the differences of the nearest neighbor label accuracies.}
    \label{fig:stat_test}
\end{figure}

\subsection{Nonlinearity of outer layers captured with relative accuracy of cluster-membership based on different distance measures}
\label{sec: stat-1}
In figure~\ref{fig:stat_test}, We first show the statistical evidence that further strengthens our claim that "the dense regions that are well-separated, while the surrounding outer
layers exhibit more non-linear structure". 

Specifically, we look at the cluster membership of some $q$ nearest neighbors of points in the inner and outer layers via different notions of distances. To get a high resolution understanding, we focus on the same pairs of clusters that we used to observe the decaying performance of K-Means as points from outer layers were considered. 

First, we look at the q nearest neighbors of core points among other core points. We consider the Euclidean distance and the shortest path on the K-NN graph distance. We record the difference in the fraction of intra-cluster points among nearest points according to different metrics. We observe that the accuracy for these two metrics remain relatively unchanged. 

In comparison, we observe that if we look at the nearest neighbors of the outer-most layer among the cores, the intra-cluster accuracy according to Euclidean distance is relatively lower compared to shortest path on the CDNN graph that we generate. This implies, that the cluster-membership of the core-core points are well captured by Euclidean distance. In contrast, we need a more local-definition of distance (via shortest path on the CDNN graph) to get cluster-membership-preserving notions of distance. 

\subsection{Inversion in the behavior of K-Means objective value}
We conclude this part with an interesting Phenomena on the K-Means objective value of the clustering obtained by K-Means and CoreSPECTed-K-Means. Essentially, we have the following observation.

\begin{enumerate}
    \item  {\bf On the cores:} The K-Means objective value of the clustering obtained by K-Means on the core is lower than the K-Means objective value of the core points when K-Means was applied to the whole dataset. 

    \item {\bf On the entire dataset:}  Here, the K-Means objective value of CoreSPECTed-K-Means is higher than that of K-Means, even though the final clustering by CoreSPECTed-K-Means is of much higher quality, as we have observed. 
\end{enumerate}

We capture this in Figure~\ref{fig:K-means-obj}. This further strengthens the assumptions that the cores obtained by FlowRank are geometrically well-separated, contrasting the more complex geometry of the outer layers.

\begin{figure}[t]
    \centering
    \includegraphics[width=\linewidth]{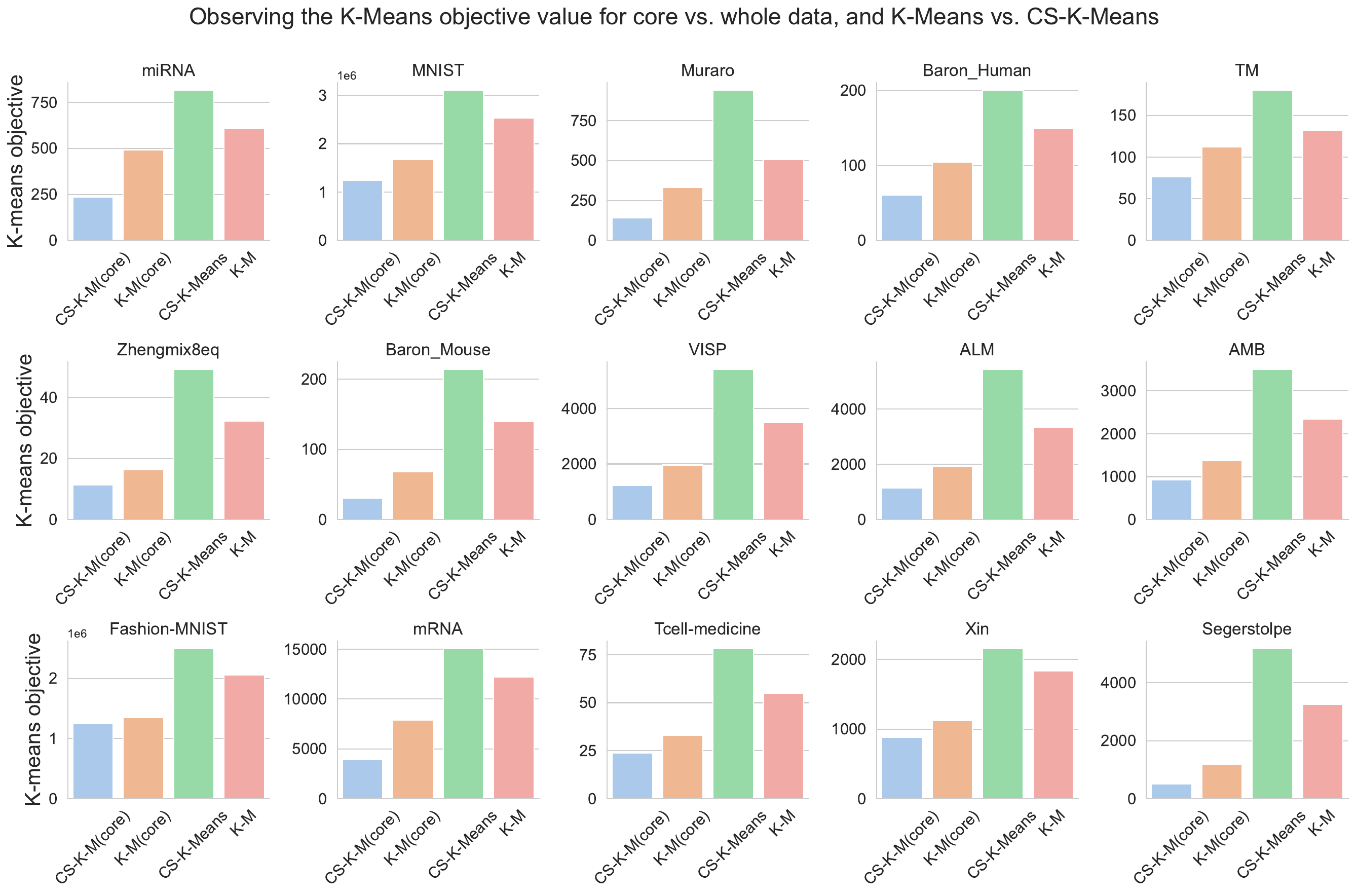}
    \caption{We observe that the K-Means objective of the clustering due to CS-K-Means is higher than that of K-Means when looking at the whole data. This contrasts with CS-K-Means having significantly better accuracy. This implies, for the whole dataset, linearly geometric objectives like that of K-Means is not suitable. However, if we focus only on the cores, a different picture emerges. The K-Means objective of the clusters among the core obtained by directly applying K-Means on core is smaller than that of the core-points when K-Means is applied on the whole dataset. This further strengthens the assumptions that the cores obtained by FlowRank are geometrically well-separated, contrasting the more complex geometry of the outer layers.}
    \label{fig:K-means-obj}
\end{figure}

\section{Ablation Studies}
\label{app:ablation}

In the previous section we have demonstrated that our framework elevates the performance of both K-Means an CS-HDBSCAN across 19 datasets, often outperforming recent density-peak-based and manifold-based clustering algorithms. In this section we provide several ablation studies, directed at understanding the importance of each of the steps in our framework. We start by recollecting the different steps in our framework on a high level in Figure~\ref{fig:app-CoreSPECT}

%We start by recollecting the four steps of our framework. First, we divide the data points into some layers, from most core to most periphery using our algorithm FlowRank. Then, we cluster the cores using the user's algorithm of choice. Next, we cluster the rest of the points by first creating the CDNN graph and then layer-wise expansion. In this section we provide an ablation study for each of the steps, to understand the importance of each of the steps. For the purpose of exposition, we fix K-Means as the user's algorithm of choice Recall that we call the application of our framework to K-Means as CoreSPECTed-K-Means (CS-K-Means in short).

\begin{figure}[ht]
\centering
\noindent\fbox{\begin{minipage}{\textwidth}
\begin{enumerate}
    \item {\bf Step 1: Core (and subsequent density layers) extraction.}

    In this step we want to obtain the density layers in the data, such that the top-most layer contains the most separable cores and the subsequent layers contain more complex and less separable peripheral points. 

    As we have described in the main paper, we design an algorithm that we call FlowRank and then define the deciles as the layers. We aim to find the relatively dense regions of the data to obtain the cores of each cluster. Our method is inspired by the concept of relative centrality in \cite{mcpc}.

    \item {\bf Step 2: Clustering the core.} 
    
    In this step, the cores ($S_0$) are clustered with a simple algorithm of user's choice. We use K-Means for this part.

    \item {\bf Step 3 and 4: Expanding the clustering to the rest of the points}

    We do this in two steps.

    \subitem {Step 3: Core directed nearest neighbor (CDNN) graph construction.}  Here, we design a layer-by-layer nearest neighbor graph structure using the estimated density layers.

    \subitem {\it Step 4: Layer-wise expansion of clustering. } 
    Finally, we use the clustering of the cores and the CDNN graph to cluster the rest of the points.
    \ \\
    
    The key idea here is that our model dictates that the cluster membership of a point in a density layer is best determined to its proximity to points in an inner layer (as opposed to overall proximity to points).

\end{enumerate}
\end{minipage}}
%\begin{center}
   \caption{ The CoreSPECT Framework (on a high level)}
   \label{fig:app-CoreSPECT}
%\end{center} 
\end{figure}

\subsection{Approximating the density layers: Comparing FlowRank with relative-centrality methods.}

\begin{figure}[t]
    \centering
    \includegraphics[width=\linewidth]{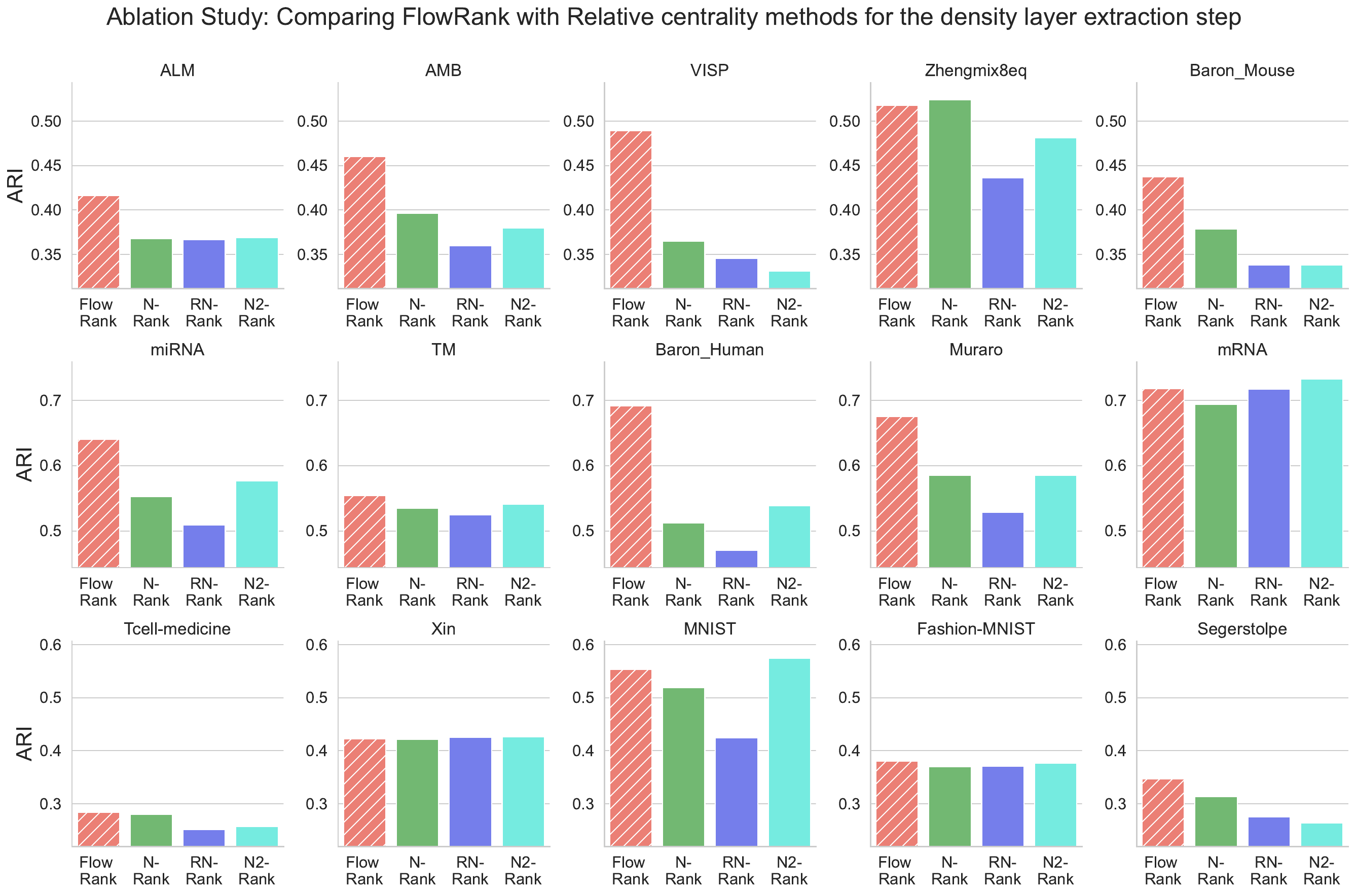}
    \caption{Contrasting the performance (ARI) of FlowRank with that of Relative centrality methods in the density-layer-extraction step of CoreSPECT. In 11 out of 15 datasets, FlowRank produces the best result, and is within 2\% of the best in the other four datasets.}
    \label{fig:ablation-ranking}
\end{figure}

We have discussed in Section~\ref{sec:small-datasets} that for some the small datasets, FlowRank fails to add points from the smallest clusters to the cores even when they are well-separated, and identified it as avenue for improvement. In this direction, we compare the merits of FlowRank w.r.t. the concept \emph{relative-centrality}~\cite{mcpc} from where we drew inspiration. In \cite{mcpc}, the authors focus on K-NN graph embeddings of biological datasets, and then define ranking algorithms such that 

i) Top ranked points are better separable into their ground truth clusters compared to the whole data. They verify separability based on what fraction of edges in the induced subgraph of the top-ranked points are \emph{intra-community}. 

ii) The top ranked points contain points from all clusters

They verify this property for 11 single-cell datasets (that we also use in our paper). The two properties sought by \cite{mcpc} seems to fit our needs exactly. That is, if the top-ranked points are more separable and we have points from all underlying ground truth clusters, we will be able to obtain a high quality clustering of these points. However, crucially, we have a more specific requirement, which is {\bf the core points should be geometrically well separated}. Note that this is not automatically ensured by the relative centrality methods. Indeed, a look at the algorithms in \cite{mcpc} suggests that they aim to obtain locally-high-centrality points. While these points themselves may be better separable, they may be from very different parts of the space, and therefore, the performance of simple Geometric clustering algorithms like K-Means may not be good. To further investigate this hypothesis, we run our experiments by replacing FlowRank with algorithms of \cite{mcpc}. We capture the results in Figure~\ref{fig:ablation-ranking}.

We observe that FlowRank usually leads to the best result, with significant advantage over the relative centrality methods in 10 out of 15 datasets, while being very similar to the best relative centrality method per-dataset for each of the other 5 datasets. 

Essentially, the reason behind this is that the relative centrality methods treat the K-NN graphs to have a single-layer core-periphery structure. However, we find that there are multiple density layers in the real world data, which necessitates our algorithm. However, we believe it is possible further improve our ranking and overall density-layer estimation algorithm, which we consider it as an important future direction.

\begin{figure}[t]
    \centering
    \includegraphics[width=\linewidth]{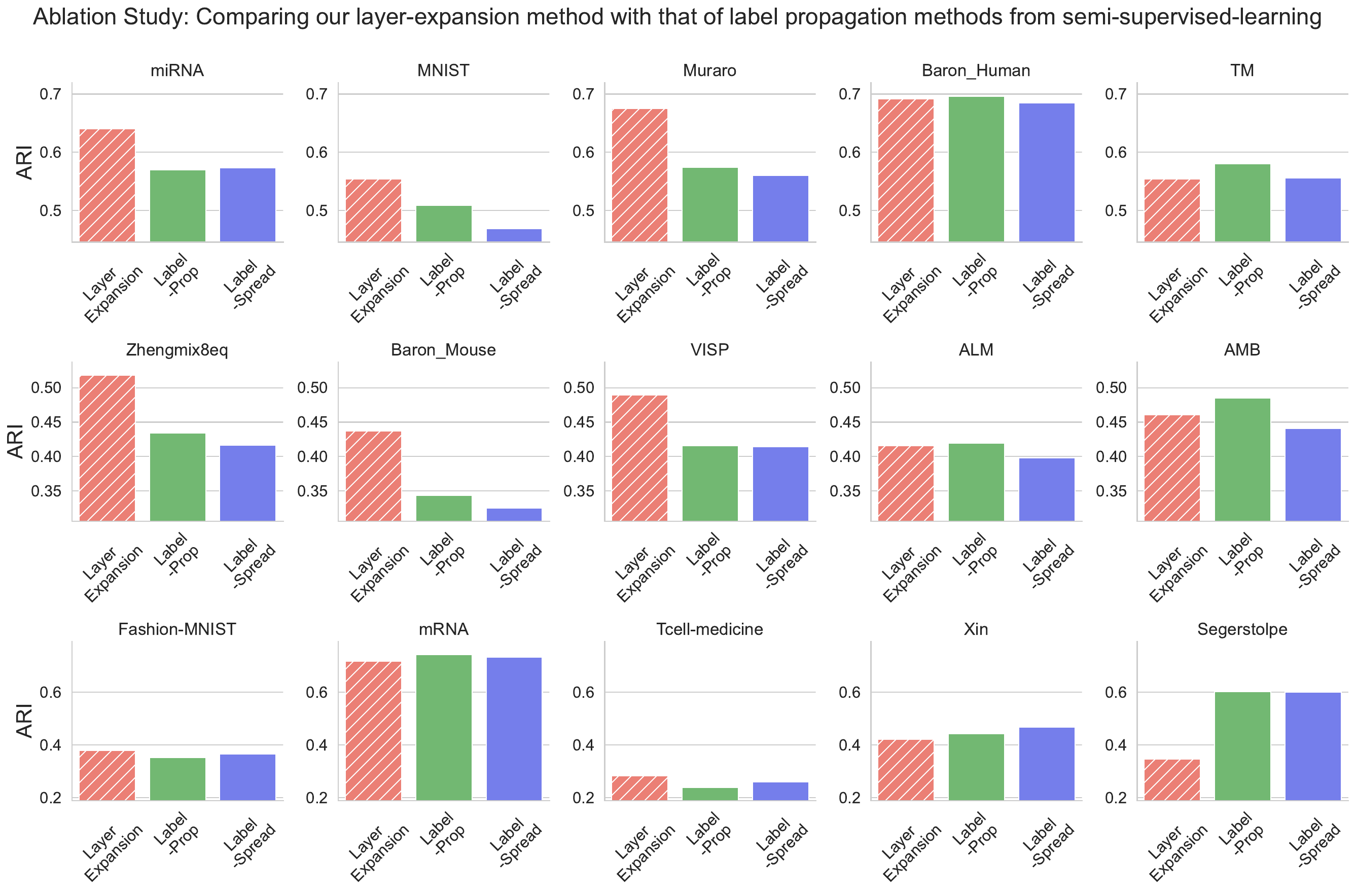}
    \caption{Contrasting the performance of CDNN graph generation+layer expansion with that of popular label propagation algorithms used in Semi supervised learning. Overall, our CDNN-graph based approach has the best rank, with noticeably better performance in 6 out of 15 datasets. We have almost identical performance in 8 more datasets, with doing noticeably worse in only one dataset (Segerstolpe).}
    \label{fig:ablation-propagtaion}
\end{figure}

\subsection{Layer-wise propagation: A connection (and comparisons) of layer-wise-expansion to semi-supervised learning.}

Next, we discuss a very obvious yet interesting connection between our framework and \emph{semi supervised learning}.  

In our framework, after we extract the cores and cluster them (which we expect and observe to have high correctness), we use this clustering to extend this to the rest of the data. This is very similar to the concept of \emph{label propagation}~\cite{label-prop,label-spread}, which is one of the fundamental approaches to semi-supervised learning (SSL). In SSL, given the true labels of some of the points, one wants to label the rest of the points. Therefore, our algorithm can be considered as some sort of a pseudo-semi-supervised-learning framework. This raises the following question. 

\begin{ques}
Can SSL based label-propagation methods extend the clustering of the core as well as our CDNN graph based approach?
\end{ques}

In this direction, we look at popular two scikit-learn method, called Label-propagation~\cite{label-prop} and Label-spreading~\cite{label-spread}. The first is a ``hard-clamping'' method. That is, the initial labels provided by to the algorithm are never altered. In contrast, Label-spreading is a ``soft-clamping'' method, where it may readjust some of the initial labels. In short, both the methods aim to create some affinity matrix from the labeled points to the rest of the data. In this direction, the original papers defined an rbf-based kernel for this purpose, which is also the default setting in the scikit-learn implementation. We first tried these two methods to enhance the clustering of the cores by K-Means. Here, we observe that both for label-propagation and label-spreading methods fail, resulting in significantly NMI and ARI compared to K-Means on the whole dataset. 

\paragraph{K-NN-based SSL label-propagation algorithms}
To further investigate the possibility of using SSL based prorogation, we use the alternative kernel that uses only nearest neighbors to propagate the initial labels, performing this recursively until all the points are labeled. We place the ARI values of using Label-propagation and label-spread instead of our layer-expansion idea in Figure~\ref{fig:ablation-propagtaion}.  

We observe that for this setting, both SSL methods perform well,  resulting in ARI that is higher than K-Means on the whole dataset in most cases. Overall, our CDNN-based approach still performs the best (except in the Segerstolpe dataset, where both the SSl methods have much better performance). This is a reasonable outcome, as for datasets that are good fit for our density-layer-structure, the cores obtained by FlowRank are the points in the inner-most layers. In contrast, in SSL one often wants to labels from \emph{all regions} of the data, which SSL based labale propagation ideas aim to implicitly exploit.

Ignoring the Segerstolpe dataset (which as we have discussed, does not contain a strong density-layer structure), our approach leads to roughly 7\% higher ARI than label-prop and 8\% higher than label-spread. On one hand, this further justifies our model formulation and algorithmic framework. On the other hand, this opens up the possibility of further improving our framework by incorporating ideas from more advanced SSL ideas.

We believe better understanding the applicability of such label propagation ideas and appropriating for our framework may lead to even better performance, which is an exciting future direction.

\begin{figure}[t]
    \centering
    \includegraphics[width=\linewidth]{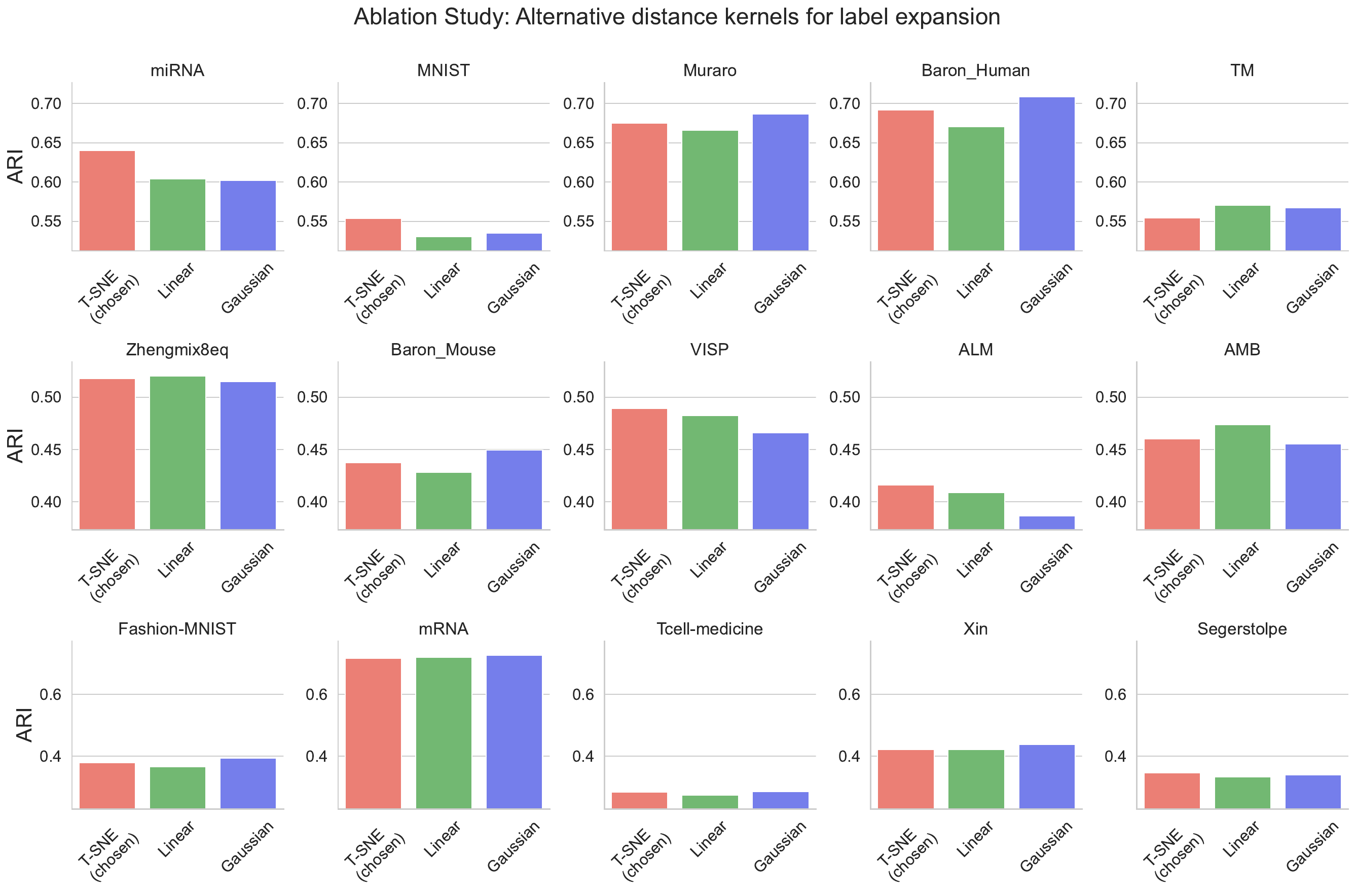}
    \caption{Performance of CoreSPECTed-K-Means for different choice of weight matrix for the layer-expansion}
    \label{fig:ablation-distance-kernel}
\end{figure}

\subsection{Choice of weight matrix in the layer-expansion approach.}

Next, we briefly recall the structure of our layer-expansion procedure. After we have partitioned the points into deciles (using the FlowRank ranking), for any point in the $j$-th layer, we find nearest neighbors among the points upto $j-1$-th layer (which have been already labeled) and then decide a (normalized) weight function on these edges such that the cluster membership vector of the new point is a weighted sum of the cluster membership vectors of its neighbors. In the results in our main paper, we have used the weight function of UMAP~\cite{mcinnes2018umap} (which itself is inspired from t-SNE~\cite{TSNE}) to decide these weights. 

Here we perform an ablation study to understand the variability in the performance of CoreSPECT based on the weight function itself (using K-Means as the core-clustering algorithm). We describe the methods we use below.

\begin{enumerate}
    \item {\bf Chosen setting:} Given a points $\mathbf{x}$ and some $t$ neighbors $Z_t$, the weight of an edge $\mathbf{x} \to \mathbf{x'}$ is decided as
    $W(\mathbf{x},\mathbf{x'})= -exp \left( \left(\|\mathbf{x} - \mathbf{x'}\|- 
    \min_{\mathbf{u} \in Z_t} \| \mathbf{x} - \mathbf{u}\|
     \right)\Big/ \sigma_x  \right)
    $
    such as the sum of all weights is a fixed constant (this is achieved by choosing a different $\sigma_x$ for each $\mathbf{x}$. 
    
    \item {\bf Linear kernel:} 
    In this setting, we simply choose
    $
    W(\mathbf{x},\mathbf{x'})= \frac{1}{1+\|\mathbf{x} - \mathbf{x'} \|}
    $

    \item {\bf Global Gaussian kernel:} This is similar to our chosen setting, without choosing a different $\sigma_x$ for each point. 
    $
    W(\mathbf{x},\mathbf{x'})= -exp \left(\|\mathbf{x} - \mathbf{x'}\|- 
    \min_{\mathbf{u} \in Z_t} \| \mathbf{x} - \mathbf{u}\|  \right)^2
    $
\end{enumerate}

In Figure~\ref{fig:ablation-distance-kernel}, we present the performance of the different distance kernels. We observe that overall, the t-SNE kernel has very slight advantage over the two methods, being better than linear kernel in 10/15 datasets, and Gaussian kernel in 8/15 datasets. However, we note that the overall differences in the three methods are negligible for most datasets.

\section{Detailed Model and Complete Proofs}
\label{app:theory}

In this Section, we will describe our Layered-core-periphery-density-model ${\sf LCPDM}(k,\ell)$ model in further detail and provide initial theoretical support to some of the steps described in the main paper. In doing so, we will also obtain some insights into the different parameters of our algorithm. First we describe the configuration of our model in detail. 

\subsection{The \textsf{LCPDM}(k,\(\ell\)) model}
\label{app:detailed-model}

As we have described, we have some assumptions on the density and some on the geometry. Here, we describe them in greater detail. We assume points of each ground-truth cluster $i$ is being generated from some $\X_i \in \mathbb{R}^d$. We make the following assumptions on the geometry-density interaction. For the rest of the discussion we focus on the case of $k=2$. First, we formally define the reach of a manifold, which we use to derive different results.

\begin{definition}[\cite{federer1959curvature}]
\label{def: reach}
The reach of a set $\mathcal{S}$ is defined as the infimum of distances between points in S and points in its medial axis, the points in
ambient space for which there does not exist a unique closest point in $\mathcal{S}$.    
\end{definition}

\paragraph{Detailed model formulation}

\begin{enumerate}[wide, labelindent=7pt]

\item {\it Globally flat manifold:} 
We assume the sets $\X_0$ and $\X_1$ lie on a flat $m$-dimensional smooth manifold $\M$, i.e., $\X_0 \cup \X_1 \subset \M$ with a reach of $\tau$ for a fixed $m$.

\item {\it Concentric sub-manifolds (with well-defined width):} The fundamental assumption we make is that $\X_i$ can be expressed as a hierarchy of concentric sub-manifolds $\X_{i,0} \subset \X_{i,1} \subset \cdots \X_{i,\ell-2} \subset \X_{i,\ell-1} =\X_i$. Recall that we define 
$\Layer_{i,j}=\X_{i,j} \setminus \X_{i,j-1}$.

These concentric sub-manifolds are then defined in the following way.
We start with a $\X_{i,0}$ with reach $\tau_{i,0}$. Then, the super-sets are defined as 
\[
\X_{i,j+1}=\{x: x \in \M, d_{\M}(x,\X_{i,j})<f_{i,j}(x)\}
\]
where $f_{i,j}: {\M} \rightarrow \mathbb{R}$ is a smooth function with $\Delta^{\min}_{i,j} \le f_{i,j}() \le \Delta^{\max}_{i,j}$, defining the minimum and maximum width of the layer $\Layer_{i,j}$. Furthermore, Each layer $\Layer_{i,j}$ has a reach of $\tau_{i,j}$.
\end{enumerate}

{\bf Density assumptions: } Each community has a density hierarchy, with the most densest layer being the central one, and the layers being progressively less dense. This is realized as follows. The data is generated by sampling $n_{i,j}$ points from each layer $\Layer_{i,j}$ uniformly at random, such that there exists a  constant $1<C<2$ satisfying 
    \[\frac{n_{i,j}}{Vol(\Layer_{i,j})} = C \cdot \frac{n_{i,j+1}}{Vol(\Layer_{i,j+1})}
    \] 
We denote the points generated from $\Layer_{i,j}$ as $\hL_{i,j}$ and together let the points be called $\hat{X}$.

{\bf Geometric assumptions w.r.t. ground truth cluster}
\begin{assumption}
\label{assume:cluster-distance}
\begin{enumerate}[wide, labelindent=5pt]
    \item {\it Cores are well-separated.} 
    The first assumption on the geometry dictates that the core-layers of the clusters are well separated in an Euclidean sense. For any clusters $i,i'$
    \[ \exists \mu<0.5 \quad \text{ s.t. } \underset{\x \in \Layer_{i,0}, \xp \in \Layer_{i,0}}{\max} \|\x-\xp \| \le \mu \cdot \underset{\x \in \Layer_{i,0}, \xp \in \Layer_{i',0}}{\min} \| \x-\xp \| \]
    
    Furthermore, we assume $\underset{\x \in \Layer_{i,0}, \xp \in \Layer_{i',0}}{\min} \| \x-\xp \| \le 0.5\tau$. This is to enforce the flatness of the underlying manifold. 
    However, note that the individual layers themselves may have smaller reach, and thus be harder to cluster using simple algorithms like K-Means.

    \item {\it Layer-wise clustering membership alignment.}
    There exists $\delta<1$ such that for any $\mathbf{x} \in \Layer_{i,j}$ and any $i' \ne i$,
    $
    \min_{\mathbf{x'} \in \Layer_{i,j-1}} \|\mathbf{x}-\mathbf{x'}\| \le \delta \cdot \min_{\mathbf{x''} \in \Layer_{i',j-1} } \|\mathbf{x}-\mathbf{x''}\|
    $.
\end{enumerate}
\end{assumption}

In this direction, we prove three results to provide initial theoretical support to our algorithm. The consistent performance of our framework warrants a more in-depth analysis and understanding and refinement of the model to further understand and exploit these underlying structures, and and we end the Section with a discussion on several such directions. 

Against this backdrop, we prove three results. First, we prove that the cores can be clustered easily.

\vspace{1em}
\subsection{Clustering the cores}
\label{sec:proof-K-Means}

Here we formally define and prove Proposition 1 of Main paper. In fact, we show that for our definition, we do not need a multi-step K-Means++. Rather, we run \emph{one-step} K-Means++ $\Theta(\log n)$ times, and select the outcome with the least K-Means objective score.

First we define the K-Means++ method of \cite{k-means-plus-plus} in brief. In the case of $k=2$, the centers are chosen as follows. The initial centroid is chosen randomly. Then, the second center is chosen as follows. 
Any datapoint $u$ is chosen to be the second center with probability
$\frac{\|u-c_1\|^2}{\sum_{v \in Y}\|v-c_1\|^2}$. That is, points are chosen with probability proportional to the squared distance to the center. Once the centers are chosen, the usual iterative K-Means algorithm is applied up to some step (or convergence).

\begin{prop*}[Restated Proposition~\ref{prop:k-means-separate}]
\label{sec:prop_one}
Let $n_0$ and $n_1$ points ($n_0>n_1$ WLOG) be sampled from  $\X_{0,0}$ and $\X_{1,0}$ respectively such that $\mu n_0 \le n_1$. These points are denoted $\hL_{0,0}$ and $\hL_{1,0}$. Consider the slightly modified K-Means algorithm. Obtain two centers using the K-Means++ method and run one-step K-Means. Repeat this process some $2 \log n$ times and accept the result with minimum K-Means objective value. This algorithm separates the two cores correctly with probability $1-o(1)$.
\end{prop*}

\begin{proof}
We condition on the event that the first center selected is from $\hL_{0,0}$. This happens with probability $\ge 0.5$ (as it is the larger set).

Let, the minimum inter-community distance be $\alpha$. Then, the maximum  intra-community distance is $\mu \cdot \alpha$.

Then, the probability that the second center is a point from $\X_{1,0}$ is
\begin{align*}
\Pr (c_2 \in \X_{1,0}) &= \frac{\sum_{\x \in \hL_{1,0}} \|c_1-\x\|^2}{\sum_{\x \in \hat{X}} \|c_1-\x\|^2}
\\ 
&= \frac{\sum_{\x \in \hL_{1,0}} \|c_1-\x\|^2}{\sum_{\x \in \hL_{0,0}} \|c_1-\x\|^2 + \sum_{\x \in \hL_{1,0}} \|c_1-\x\|^2}
\\ 
&\ge 
\frac{\alpha^2 \cdot n_1}{\mu^2\alpha^2 n_0 + \alpha^2n_1}
\\ 
&\ge 
\frac{\alpha^2 n_1}{\mu \alpha^2 (\mu n_0) + \alpha^2 n_1}
\\ 
&\ge 
\frac{\alpha^2 n_1}{\mu\alpha^2 n_1 + \alpha^2 n_1}
\ge \frac{1}{\mu+1}
\end{align*}

That is, with probability $\frac{1}{\mu+1}$, the two centers chosen belong to the different ground truth clusters. Let $1_i$ be the indicator variable for the $i$-th initialization that is $1$ if the two centers are from two clusters. Then, we have
\begin{align*}
&\Pr \left( \sum_{i=1}^{2 \log n} 1_i >0 \right)= 1- (\Pr(1_i=0))^{2 \log n}
\ge 1 -\left( \frac{2\mu+1}{2\mu+2}\right)^{2 \log n}
=1-o(1)
\end{align*}

This shows that with high probability, one of the initialization will lead to two centers being chosen from two different ground truth clusters. Then, based on the distance definition, clustering the points by assigning each point to its closest center will lead to a solution with zero misclassification error. 

Next, note that the K-Means objective of any solution with two centers being selected from the different ground truth cluster is smaller than the K-Means objective of any solution where both centers are selected from the same cluster. This is because of the following. 

{\bf Case 1:} The K-Means objective value if the two centers are selected from different clusters is upper bounded by $(\mu \alpha)^2 (n_0+n_1)$. This is assuming each point is farthest possible away from the centers.

{\bf Case 2:} The K-Means objective value if both centers are selected from the same cluster is lower bounded by $\alpha^2 n_1$. This is assuming both centers are in the larger cluster $X_{0,0}$ and the K-Means objective value w.r.t. the points in $\X_{0,0}$ is zero.

Then, we have 
\[
(\mu \alpha)^2 (n_0+n_1) \le \mu \alpha^2 \cdot \mu n_0 + \mu^2 \alpha^2 n_1
\le  \mu \alpha^2 n_1 +  \mu^2 \alpha^2 n_1 \le (\mu+\mu^2)\alpha^2n_1 
\le 0.75\alpha^2n_1
\]

This implies that the maximum K-Means objective value when the centers are chosen from different ground truth clusters is smaller than the minimum K-Means objective when the two centers are chosen from the same ground truth cluster. This, combined with the fact that with probability 1-o(1) one of the initializations selects centers from different ground truth clusters completes the proof.

\end{proof}

\subsection{Extracting the cores with FlowRank}
\label{sec:proof-FlowRank}

Next, we prove that if the FlowRank algorithm were given the exact density value for each point (which we have assumed to be same for any Layer $\Layer_{i,j}$), FlowRank gives a score of $1$ to all core points (points sampled from $\Layer_{i,0}$) and a score of $<1$ to all non-core points. This gives us an initial insight into FlowRank's functionality. If all the core points are valued $1$ and all non-core points are valued at less than $1$, then selecting all the $1$-valued points will give us the cores, which can then be clustered as shown in Section~\ref{sec:proof-K-Means}.

\subsubsection{Conditions for core-score being 1}
Recall that in FlowRank, after we have a density estimation of the points, we have an ascending random-walk step such that for each point $u$, we look at its some $r=\mathcal{O}(1)$ neighbors, and move to a neighbor with higher density value at random. We continue this process until we reach a maxima, and the FlowRank value of $u$ is the ratio of the density of $u$ and the average densities of the reachable density peaks. Then, we have the following argument.

\begin{lemma}
\label{lem:core-condition}
If for any core point $\x \in \Layer_{0,0}$, all of its $r$ neighbors are in $\X_0$, the FlowRank score of $\x$ is $1$. Furthermore, in the general case, this is a necessary condition. 
\end{lemma}
\begin{proof}
This is because, as long as all the neighbors of $\x$ are in $\Layer_{0,0}$, they are either another point in $\Layer 
_{0,0}$ with same density as $\x$ or a point in $\X_{0,j},j>0$ with lower density. following the formulation of the model. This makes $\x$ a local maxima of the walk itself. That is, any density-ascending random walk cannot take any step starting from $\x$, resulting in FlowRank score of $\x$ being 1. 

The necessity of this condition follows from the fact that we do not make any assumptions on the relation of densities of the layers from different clusters. Even if $\Layer_{0,0}$ is the highest density region in $\X_0$, its density may be lower than the lowest density region in $\X_{1}$. In such a case, if there is any of the $r$ neighbors of a point in $\Layer_{0,0}$ is from $\X_1$, that point will get a FlowRank score of less than $1$.

\end{proof}

Then, we want to prove that this condition indeed holds in our model. We first define some notations.
\begin{definition} 
\label{def: notation}
We define 
$\sup\dis(S)$ as the maximum Euclidean distance between two points in a set $S$. 
Similarly, we define $\sup\dis(S,X,r)$ as the maximum distance to the $r$-th nearest neighbor of points in $S$ in the Set $X$.
\end{definition}

%Then, for any $r=\mathcal{O}(1)$, $\sup\dis(\Layer_{0,0})>\sup\dis(\Layer_{0,0},\X,r)$ for a sufficiently large $n$. Then we immediately have that in the c
In this direction, we first show that inter-core edges are impossible.

\begin{Prop}[Inter-core edges are impossible]
\label{prop:core-to-core-edge}
Let $\x$ be a point in $\Layer_{0,0}$. Then for any $r=\mathcal{O}(1)$, none of $\x$'s $r$-closest neighbors are from $\Layer_{1,0}$.
\end{Prop}
\begin{proof}
From the first condition in Assumption~\ref{assume:cluster-distance}, we know that the maximum intra-core distance is $\mu$ times smaller than the minimum inter-core distance. As there are $\omega(r)$ many points in $\Layer_{0,0}$, none of its $r$ closest neighbor can be from $\Layer_{1,0}$.

\end{proof}

Next we want to prove that none of the $r$ neighbors of $\x \in \Layer_{0,0}$ is in $\Layer_{1,j},j>0$. We show that for the case of $j=1$, the distance property of our model suffices.

\begin{Prop}
\label{prop:core-to-second-layer}
Let $\x$ be a point in $\Layer_{0,0}$. Then for any $r=\mathcal{O}(1)$, none of $\x$'s $r$-closest neighbors are from $\Layer_{1,1}$
\end{Prop}
\begin{proof}
Let $\x$ be any point in $\Layer_{0,0}$. Let $\xp$ be $\x$'s nearest neighbor (in Euclidean distance) from $\Layer_{1,1}$.

Let this distance be $d_{0,1}$. This implies $\min_{\xpp \in \Layer_{0,0}} \|\xp-\xpp\|<d_{0,1}$. This implies $\min_{\xpp \in \Layer_{1,0}} \|\xp-\xpp\| < \delta \cdot d_{0,1}$ From the second condition of Assumption~\ref{assume:cluster-distance}.
Then, by triangle inequality we have 
\begin{align*}
&\min_{\xpp \in \Layer_{1,0}} \|\x-\xpp\| \le (1+\delta) d_{0,1}
\\ \implies &
(1+\delta) d_{0,1} \ge \frac{1}{\mu}\cdot \sup\dis(\Layer_{0,0}) \hspace{1em} \text{[From Condition 1 of Assumption~\ref{assume:cluster-distance}]}
\\ \implies &
d_{0,1} \ge \frac{1}{\mu(1+\delta)} \sup\dis(\Layer_{0,0})
\\ \implies &
d_{0,1} \ge \sup\dis(\Layer_{0,0})
\end{align*}

This completes the proof.

\end{proof}

Next, we want to extend the proof to the further outer layers $\Layer_{1,2} , \ldots, \Layer_{1,\ell}$. This is more complicated, as the data lies on some manifold, with individual layers having more complex structure. Here we use the fact that the underlying Manifold $\M$ has a large reach, and get a bound on the distortion of Euclidean and Geodesic distance. 

\begin{theorem}[Lemma 3 of \cite{boissonnat2019reach}]
\label{thm:distortion}
Let $S \subset \mathbb{R}^d$ be a closed set with reach $\tau$, as defined in Definition~\ref{def: reach}. Then for any $a,b \in S$ such that $\|a-b\|<2\tau$, one has $\frac{d_S(a,b)}{\|a-b\|} \le \frac{2\tau}{\|a-b\|} \cdot \arcsin \frac{\|a-b\|}{2\tau}$.
\end{theorem}
Here $\|a-b\|$ is the Euclidean distance between $a,b$ and $d_S(\cdot,\cdot)$ is the Geodesic distance on $S$ (for our application $\M$). We shall use this to prove that none of the r nearest neighbors of $\Layer_{0,0}$ are in $\X_1$. Furthermore, note that $\alpha \arcsin \frac{1}{\alpha} < \pi/2$ for $\alpha> 1$.

\begin{lemma}
\label{lem:core-gets-1}
Under the constraints of our model and $r=\mathcal{O}(1)$, if $\mu \le ({\frac{1}{c \cdot(1+\delta)})^{\ell-1}}$ then none of the $r$-nearest Euclidean neighbor of a point in $\Layer_{0,0}$ lies in $\X_1$. This implies that the FlowRank score of any core point is $1$.
\end{lemma}

\begin{proof}
We prove this with an induction on the minimum distance of $\x \in \Layer_{0,0}$ to $\Layer_{1,j}$. For any such $\x$, let $d_{0,j}$ be the minimum distance from $\x$ to $\Layer_{1,j}$. 
In proposition~\ref{prop:core-to-second-layer} we have shown that $d_{0,1} \ge \frac{1}{\mu(1+\delta)} \sup\dis(\Layer_{0,0})$. Then, our base case is as follows. 

Let $d_{0,2}$ and $\tilde{d}_{0,2}$ be the smallest distance of $\x$ to $\Layer_{1,2}$ in the Euclidean and geodesic distance on $\M$, respectively. First, note that $1/c \cdot \tilde{d}_{0,2} \le d_{0,2} \le \tilde{d}_{0,2}$ for some $1<c<\pi/2$. The lower bound follows from the distortion bound of Theorem~\ref{thm:distortion} and the upper bound follows from the fact that geodesic distance between two points is always larger than the Euclidean distance.

Let the end points be $\y$ and $\yt$ respectively. Then, there exists a point $\tilde{\x} \in \Layer_{0,1}$ such that $\dis_{\M}(\tilde{\x},\yt) \le \tilde{d}_{0,2}$. This is because, if we consider any geodesic path starting from $\Layer_{0,0}$ that goes out of $\X_1$, it has to go through $\Layer_{0,1}, \ldots ,\Layer_{0,\ell}$. This is because the concentric subspaces are defined by expanding the smaller subspaces in all directions within the Manifold $\M$. Then $\|\tilde{\x}-\yt\| \le \tilde{d}_{0,2}$. Then, from the second condition of Assumption~\ref{assume:cluster-distance}, we have 
\[
\underset{\yp \in \Layer_{1,1}}{\min} \|\yp- \yt\| \le \delta \cdot \tilde{d}_{0,2}
\]

Let $\ypp$ be the point for which the minimum is achieved. Now, we use the distances between $\x \in \Layer_{0,1}$, $\yt \in \Layer_{1,2}$ and $\ypp \in \Layer_{1,1}$, we have a triangle inequality that upper bounds the minimum distance from $\Layer_{0,0}$ to $\Layer_{0,1}$, which we have defined as $d_{0,1}$ in  Proposition~\ref{prop:core-to-second-layer}. Combining, we get
\begin{align*}
&d_{0,1} \le \|\x-\yt\| + \|\yt-\ypp\|
\\ \implies &
d_{0,1} \le \tilde{d}_{0,2} + \delta \cdot \tilde{d}_{0,2} 
\\ \implies &
d_{0,1} \le (1+\delta) \tilde{d}_{0,2} 
\\ \implies &
 \tilde{d}_{0,2} \ge \frac{d_{0,1} }{(1+\delta)}
 \\ \implies &
 d_{0,2} \ge \frac{d_{0,1} }{c(1+\delta)} \hspace{5em}
\text{[From Theorem~\ref{thm:distortion}]} 
 \\ \implies & d_{0,2} \ge \frac{1}{c \mu (1+\delta)^2} \sup\dis(\Layer_{0,0})
\end{align*}

We can continue with this inductively. Let  
$d_{0,j}$ is the minimum distance from $\x \in \Layer_{0,0}$ to $\Layer_{1,j}$.
Assume $d_{0,j} \ge \frac{1}{\mu c^{j-1}(1+\delta)^j} \sup\dis(\Layer_{0,0})$. Note that this is proven for $d_{0,2}$.

Then, we can lower bound $d_{0,j+1}$ as follows. 
Consider any $\x \in \Layer_{0,0}$. Let $\yt$ be the closest point of $\x$ in $\Layer_{1,j+1}$ in geodesic distance, and the distance be $\tilde{d}_{0,j+1}$ Then, there is a point $\tilde{x} \in \Layer_{0,j}$ such that $\dis_{\M}(\tilde{x},\yt) \le \tilde{d}_{0,j+1}$. Then $\|\tilde{x}-\yt\| \le \tilde{d}_{0,j+1}$. Then, there exists $\ypp \in \Layer_{1,j}$ such that $\|\ypp-\yt\| \le \delta \tilde{d}_{0,j+1}$. Note that this is an upper bound on $d_{0,j}$ via triangle inequality. 

Then, we have
\begin{align*}
&d_{0,j} \le (1+\delta) \tilde{d}_{0,j+1}
\\ \implies & \tilde{d}_{0,j+1} \ge \frac{d_{0,j}}{(1+\delta)}
\\ \implies & 
d_{0,j+1} \ge \frac{d_{0,j}}{c(1+\delta)}
\\ \implies & 
d_{0,j+1} \ge
\frac{1}{\mu \cdot (c^j(1+\delta)^{j+1})} \sup\dis(\X_{0,0})
\end{align*}

This completes the induction. 
Then, as long as $d_{0,\ell-1} \ge \sup \dis(\X_{0,0},\X_{0,0},r)$, there are no edges from $\Layer_{0,0}$ to $\X_1$. A trivial upper bound on $\sup\dis(\Layer_{0,0},\X,r)$ is $\sup\dis(\Layer_{0,0})$ as $r=\mathcal{O}(1)$. Using this we get that if $\frac{1}{\mu c^{\ell-2}(1+\delta)^{\ell-1}}>1$, then there are no edges going from $\Layer_{0,0}$ to $\X_1$. This completes the proof.

\end{proof}

\subsubsection{Non-core points have lower score.}
Next we show there exists $r=\mathcal{O}(1)$ such that for each non-core point in $\Layer_{i,j}$, one of its $r$ nearest neighbor lies in $\Layer_{i,j-1}$. We first place some notations and results on volumes of balls on Manifolds. 

\begin{theorem}[\cite{niyogi2008finding}]
\label{thm:vol-lower}
Let $p \in \M$ which is a compact smooth $m$-dimensional manifold with reach $\tau$. Consider $A= \M \cap B_{\epsilon}(p)$ where $\epsilon \ll \tau$. Then $vol(A) \ge \cos(\theta)^m \cdot vol(B^m_{\epsilon}(p))$ where $B^m_{\epsilon}(p)$ is the $m$-dimensional ball in $T_p$ (the tangent space) centered at $p$, $\theta= \arcsin(\epsilon/2\tau)$
\end{theorem}

Furthermore, we know that $V^m(\epsilon)= B^m_{\epsilon}(p)$ where we define $V^m(\epsilon)$ as the volume of the $m$-dimensional ball of radius $\epsilon$. As we need a more intricate relationship between the width of the layer and the points in local neighborhood, we make some assumptions about the non-linear structure of the each of the layers.

\begin{assumption}
\label{assume:last-proof}
\
 \begin{enumerate}
     \item We assume that all layers have the same width $\Delta$ and the reach of each sub-manifold $\X_{i,j}$ is $\bar{\tau}$ such that $\Delta \ll \bar{\tau} \ll \tau$. Furthermore, assume $\ell \le m$. Furthermore, note that we do not any specific reach assumptions for the outermost layers.

     \item Let there be $n_{0,0}$ points sampled in $\Layer_{0,0}$ such that $n_{0,0} = c_2\frac{Vol(\X_{0,0})}{V^m({\epsilon})}$ for some $\epsilon \le c_3\Delta$ where $c_2>1$ and $c_3<0.001$ are constants. This, combined with Theorem~\ref{thm:vol-lower} essentially implies that any $\epsilon$-radius ball in $\Layer_{0,0}$ has some $c_2$ points on expectation.

 \end{enumerate}    
\end{assumption}

Based on these assumptions, we have the following notion. 
\begin{Prop}
\label{prop:layer-density}
For any Layer $\Layer_{0,j}$. Then, the number of points $n_{0,j}$ sampled by the model satisfies $n_{0,j} \ge \frac{Vol(\Layer_{0,j})}{V^m(2\epsilon)}$
\end{Prop}
\begin{proof}
In the definition of our model we have
$\frac{n_{0,0}}{Vol(\Layer_{0,0})} = C \cdot \frac{n_{0,1}}{Vol(\Layer_{0,1})}$.
Then, replacing with first condition of Assumption~\ref{assume:last-proof} we get
\[
C \cdot \frac{n_{0,1}}{Vol(\Layer_{0,1})} \ge \frac{c_2}{V^m(\epsilon)} \implies 
n_{0,1} \ge \frac{c_2 Vol(\Layer_{0,1})}{C \cdot V^m(\epsilon)}
\]
Continuing this for some $j$-steps we get $n_{0,j} \ge 
\frac{c_2Vol(\Layer_{0,j})}{C^j \cdot V^m(\epsilon)}$.

However, as $C<2$ we get $C^jV^m(\epsilon) \le V^m(2\epsilon)$. This implies $n_{0,j} \ge \frac{c_2Vol(\Layer_{0,j})}{V^m(2\epsilon)}$.

\end{proof}

Then, we argue that for any point $\x \in \hL_{0,j}$, there exists a nearby point in $\hL_{0,j-1}$.
\begin{Prop}
\label{prop-1-delta}
Under the constraint of Assumption~\ref{assume:last-proof}, for any point $\x \in \hL_{0,j}$, on expectation there exists $\xp \in \hL_{0,j-1}$ such that $\|\x-\xp\| \le 1.1\Delta$.   
\end{Prop}
\begin{proof}
This follows from our Proposition~\ref{prop:layer-density}. 
Consider the closest point of $\x$ in $\Layer_{0,j-1}$. Let this point be $\yp$. Consider a point $\ypp \in \Layer_{0,j-1}$ such that $\|\yp-\ypp\|=6\epsilon$ and the minimum distance between $\ypp$ and $\Layer_{0,j'}$ for $j' \ne j-1$ is at least $4\epsilon$. Such a point can always be found following the definitions of the layers (which have a width of $\Delta \ge 100\epsilon$. Consider the ball $B_{3\epsilon}(\ypp)$. Find $\ypp$ such that $B_{3\epsilon}(\ypp) \cap (\M \setminus \Layer_{0,j-1}) =\emptyset$. Furthermore, from Theorem~\ref{thm:vol-lower} we have $Vol(B_{3\epsilon}(\ypp) \cap \M) \ge 0.99 V^m(3\epsilon)$.
Then, the expected number of points in $B_{3\epsilon}(\ypp) \cap \hL_{0,j-1}$ can be written as 
\begin{align*}
&\mathbb{E} 
\left[ 
|\hL_{0,j-1} \cap Vol(B_{3\epsilon}(\ypp) |\right]
\ge n_{0,j-1} \cdot \frac{0.99V^m(3\epsilon)}{Vol(\Layer_{0,j-1})} 
\\& \ge 
\frac{c_2Vol(\Layer_{0,j-1})}{V^m(2\epsilon)} \cdot \frac{V^m(3\epsilon)}{Vol(\Layer_{0,j-1})} \quad \quad \quad  \text{From Proposition~\ref{prop:layer-density}}
\\& \ge  c_2 \frac{V^m(3\epsilon)}{V^m(2\epsilon)} \ge 1
\hspace{4em}
\end{align*}

Then, the distance of this point to $\x$ is upper bounded by 
$\Delta+9\epsilon \le 1.1\Delta$.

Here the lower bound of $0.99V^m(3\epsilon)$ follows from the fact that $\cos \Theta$ in Theorem~\ref{thm:vol-lower} gets very close to 1 when radius is much smaller than reach, which is the case here as $\epsilon \le 0.01\Delta \ll \tau_{i,j}$.

\end{proof}

That is, we have proven that in our model, if you consider a radius of $1.1\Delta$, you can always find a neighbor in the inner layer. Then, finally, we are left to calculate what is the maximum possible points in this radius in our configuration.

\begin{lemma}
\label{lem:choice-of-r}
Given our model and the conditions in Assumption~\ref{assume:last-proof}, there exists $r=\mathcal{O}(1)$ such that if we consider the $r$-nearest neighbors of any point in $\hL_{0,j}$, at-least one neighbor lies in $\hL_{0,j-1}$. If our FlowRank algorithm is initialized with such an $r$, then all non-core points get a score of less than $1$.   
\end{lemma}
\begin{proof}
In Proposition~\ref{prop-1-delta} we have obtained an upper bound on the radius that ensures this behavior. We finally calculate the maximum expected number of points in this radius. Consider the densest region of $\X_0$, that is $\X_{0,0}$. Then, the number of points in a radius of $1.1\Delta$ can be upper bounded as
\[
\mathbb{E} 
\left[ 
| \hat{X} \cap B_{1.1\Delta}(\x)| \le 
n_{0,0} \cdot \frac{Vol( B_{1.1\Delta}(\x))}{Vol(\X_{0,0})}
\le c_3 \frac{V^m(1.1\Delta)}{V^m(2\epsilon)}
\right]
\]

Note that this value is exponential only in $m$, but as $m$ is a constant, we get a constant upper bound on $r$.

\end{proof}

Finally, we recall the Theorem from the main paper and combine the two results we obtained formally.

\begin{theorem*}[Restated Theorem~\ref{thm: FlowRank}]
\label{sec:theo_one}
Let data be generated from the ${\sf LCPDM}(2,\ell)$ model.
Assume we get the exact density of the points and expected value of randomly ascending random walk in the FlowRank algorithm. Let this be called the ideal FlowRank outcome. Then all the core points ($\hL_0$) get a score of $1$. Additionally, all non-core points get a score $<1$.   
\end{theorem*}

\begin{proof}
First we recall that $\Layer_{i,0}$ and $\X_{i,0}$ are synonymous. In Lemma~\ref{lem:core-gets-1} we have shown that for any point in $\Layer_{i,0}$ if the cores are sufficiently separated then all points in it (the core points) get a score of $1$.

Next in Lemma~\ref{lem:choice-of-r} we show that under the additional density conditions of Assumption~\ref{assume:last-proof}, there exists an $r$ (that depends exponentially on the dimension of the underlying manifold $\M$, which is a constant) such that for each non-core point in $\Layer_{i,j}$, one of its $r$-nearest neighbor lies in $\Layer_{i,j-1}$ in expectation, and therefore the FlowRank score is less than $1$. This completes our proof.

\end{proof}

\subsection{Correctness of CoreSPECT framework given the correct layers.}

\begin{theorem*}[Clustering in the $\sf LCPDM$ model (Restated Theorem ~\ref{thm:main-sep}]
\label{sec:theo_three}
Let $X$ be $n$ datapoints generated from the ${\sf LCPDM}(2,\ell)$ model. Let us have an estimate of the density layers, given as $S_0, \ldots ,S_{\ell-1}$ such that $|S_j \cap \hL_j|=(1-f)|\hL_j|,j>0$. for a sufficiently small function $f=o(1)$ that is layer preserving. Then, applying a variant of K-Means to $S_0$ and expanding the clustering using a CDNN graph (with correctly chosen parameters) results in clustering with o(1) misclassification error rate on expectation.
\end{theorem*}

We prove this with a minimal adjustment. We have shown in Lemma~\ref{lem:core-gets-1} that all core points get a score of $1$ and all non-core points get less than $1$ by FlowRank. Therefore we assume $S_0=\hL_0$. Furthermore, we simply set $t=1$ for the proof. The proof is then as follows. 

\begin{proof}
First from Proposition~\ref{prop:k-means-separate} we know that the $\log n$-attempt single-round K-Means separates the core with probability $1-o(1)$. We continue conditioned on this scenario.

Now, consider $S_1$. We know $|S_1 \cap \hL_1|=(1-f)|\hL_1|$. Furthermore, as $S$ is layer-preserving, we assume $|S_1 \cap \hL_3|= 0$. Otherwise, $S_1$ contains all points in $\hL_1$ and $\hL_2$ and we can treat them as a single-layer in our initial model definition (albeit with a more complex analysis). 

Then, for all points in $\hL_1$, we know the following. Let $\x \in \hL_1: \x \in \X_0$. 
Then $\min_{\xp \in \Layer_{0,0}} \|\x-\xp\| \le \delta \cdot \min_{\xp \in \Layer_{1,0}} \|\x-\xp\|$. This implies if we apply our layer-expansion with one neighbor, all such points will be correctly classified. 

Assume that in the worst case scenario we misclassify each $\x \in S_1 \cap \M \setminus \hL_1$, which is $\le f(n)$ many points, with $n$ being the total number of points. Lets call this set $E_1$.

Then, when considering the points in $S_2$, let's consider what is the criteria for this point to be correctly clustered. Consider $\x \in S_2 \cap \hL_2$. We know that on expectation there exists a point $\xp \in \hL_1$ such that $\|\x-\xp\| \le 1.1\Delta$ (from Proposition~\ref{prop-1-delta}). If $\xp$ is either correctly classified or is present in $S_1$, then $\x$ is correctly clustered. 

Therefore, the set of points incorrectly clustered from $S_2$ is upper bounded by the set of points that is within $1.1\Delta$ distance of a misclassified point or a point in $\hL_1$ that is not in $S_1$. Additionally, any point in $S_2$ that is not from $\hL_2$ may get incorrectly clustered.

Here, we note that the maximum number of points in the $1.1\Delta$-radius of any point is upper bounded by $\frac{ n_{0,0} V^m(1.1\Delta)}{Vol(\Layer_{0,0})}$ assuming they are all from maximum density points.
This can be simplified as 
\[
\frac{ n_{0,0} V^m(1.1\Delta)}{Vol(\Layer_{0,0})} \le 
\frac{c_2V^m(1.1\Delta)}{V^m(1.1\epsilon)} \quad \quad \text{[From Condition 2 of Assumption~\ref{assume:last-proof}]}
\]
which we upper bound as the constant $C_m$ (as $m$ is fixed).  Then, the total number of misclassified points in $S_2$ is upper bounded by
\begin{align*}
|E_1| \cdot C_m + |S_2 \setminus \hL_2| \cdot C_m +
|\hL_1 \setminus S_1| \cdot C_m = \mathcal{O}(f(n))
\end{align*}

We continue this for $\ell$ layers. Assume $E_j$ points have been misclassified up to layer $j$. Then the number of misclassified points in $S_{j+1}$ can be upper bounded as 
$|E_j| \cdot C_m + |S_{j+1} \setminus \hL_{j+1}| \cdot C_m +
|\hL_{j} \setminus S_j| \cdot C_m$. Now, if we induct on $|E_j|=\mathcal{O}(f(n))$ then we get $|E_{j+1}|=\mathcal{O}(f(n))$. This implies that 
$|E_{\ell}|=\mathcal{O}(f(n))=o(n)$. 

This implies we finish clustering all the points with $o(n)$ many misclassifications, which leads to an $o(1)$ misclassification error rate, completing our proof.

\begin{comment}
all points in $S_2 \cap \hL_2$ will be correctly clustered that is \emph{not} within a $1.1\Delta$-distance of a misclassified point. This is because, all the other points will have a neighbor of its own inner layer within $1.1\Delta$-distance with high probability. Then, the expected misclassification in this round is that points in $S_2$ that are not in $\hL_2$ (lets call this $E_{2,1}$ and all the points in $1.1\Delta$ radius of a point in $E_1$, Let's call this $E_{2,2}$. Then, from our definition, 
$E_{1,2}=f(n)$ and $|E_{2,2}| \le |E_1| \cdot \frac{ n_{0,0} V^m(1.1\Delta)}{Vol(\L_{0,0})} \le
|E_1| \cdot \frac{c_2V^m(1.1\Delta)}{V^m(1.1\epsilon)}$. 
This is assuming worst case scenario that all the points in $E_1$ are in high density region.

As we have discussed $\frac{c_2V^m(1.1\Delta)}{V^m(1.1\epsilon)}= \mathcal{O}(1)$, so we get $|E_{2,2}|=\mathcal{O}(|E_1|)$. This implies that a total of $|E_{1,2}|+|E_{2,2}|=\mathcal{O}(f(n))$ points are misclassified in the second round. Now, we can continue this for $\ell$ layers as follows. 

Let there be $E_{j}: |E_{j}|=\mathcal{O}(f(n))$ many misclassified points at round $j$. Then, the total number of misclassifications is upper bounded by the number of points 
\end{comment}

\end{proof}

\paragraph{Discussion of runtime.}

Finally, we note that given the CDNN graph, the run time of our layer-expansion method is $n \cdot k \cdot t$.

\begin{theorem*}[Restated Theorem~\ref{prop: time}]
\label{sec:theo_two}
Given the CDNN graph $G^+_{t,S}$ and a clustering of $S_0$, the rest of the points can be clustered in $\mathcal{O}(n \cdot k \cdot t)$ time, which is linear in the number of the edges and the number of clusters in $G^+_{t,S}$. 
\end{theorem*}
\begin{proof}
All points in $S_0$ are already clustered. Starting with $S_1$, loop through every point $\mathbf{u}$ in $S_1$ and update the cluster membership vector $\hat{C}(\mathbf{u})$ and the cluster label $k_u$ by the following:  $\hat{C}(\mathbf{u}) \gets \underset{v \in N_{G^+}}{\sum} W(\mathbf{u},\mathbf{v}) \cdot \hat{C}(\mathbf{v}) $;
    \quad $k_u \gets {\sf arg}\min \hat{C}(\mathbf{u})$. Repeat the same procedure with $S_{j+1}$ once all points in $S_j$ are clustered. Here, for each point in $S_j$, we look at the $k$-length cluster membership vector 
    $\hat{C}(\mathbf{u})$ for some $t$ many neighbors, which takes $\mathcal{O}(t \cdot k)$ time. This makes the total runtime $\mathcal{O}( n \cdot k \cdot t)$.
    
\end{proof}

\subsection{Runtime Analysis of the Complete Algorithm}
\label{sec:runtime_analysis}
First, we go through $q$-NN generation and FlowRank. Here we generate $q$-NN graph on $n$ many $d$-dimensional points using HNSW. Let’s call this runtime HNSW($n$,$d$,$q$). Then, FlowRank is obtained by running $log \ n$ step random walks and then ascending random walks (that are truncated after a $log \ n$ step). This takes $O(n log^2n)$ including $O(log\ n)$ iteration of ascending random walk from each node to estimate random walk behavior. Then we run any clustering algorithm on the top $c$ fraction of nodes (the core nodes), which takes CLUST($c \cdot n$). Once the cores are extracted and clustered, we apply propagation.  Each propagation step requires HNSW($n,d,t$) $+\ ntk$ run-time, where k is the number of clusters output when clustering the core. Then, the total run-time can be written as $O(n log^2n)\ + \text{CLUST}(c\cdot n) + \ n\cdot t \cdot k \cdot l + l \cdot \text{HNSW} (n, d, max(q,t)) $ where $l$ is the number of layers. It is well known that HNSW$(n,d,max(q,t))$ operates approximately at $O(max(q,t)\cdot n logn)$
, so the total run-time is 
$O(n log^2n\ + \text{CLUST}(c\cdot n) + \ n\cdot t \cdot k \cdot l + l \cdot (q+t)\cdot nlogn)$. This is quite fast in practice, with our framework needing around a minute to cluster CIFAR-20 (50,000 points, 768 dimension, and 20 clusters).

\section{Experiments}
\label{app:exp}

\subsection{Description of datasets and experiments}
\label{sec: exp}

As we have described, we use a total of 15 datasets in this paper. We describe the details of the datasets here. We use 11 single-cell datasets, 2 bulk-RNA datasets, and two image datasets.

\paragraph{Single-cell datasets:}
We use all of the 11 datasets from \cite{mcpc}. The details of the datasets are as follows. 
\begin{table}[H]
    \centering
    \begin{tabular}{ccccc}
    \toprule
       Name &  \# of points & \# of  communities & Source\\ \midrule
       Baron\_Human &  8569 & 14 & \citep{single-cell-7data}\\
       Baron\_Mouse & 1886 & 13 & \citep{single-cell-7data} \\
       Muraro &  2122 & 9 & \citep{single-cell-7data} \\
       Segerstolpe & 2133 &  13 & \citep{single-cell-7data}\\
       Xin &  1449 & 4 & \citep{single-cell-7data}\\
       Zhengmix8eq &  3994 & 8 & \citep{single-cell-duo} \\
       T-cell dataset &  5759 & 10 & \citep{nature-medicine-t-cell} \\
       ALM &  10068 & 136 &  \citep{single-cell-ALM-VISP}\\
       AMB &  12382 & 110 &  \citep{single-cell-7data}\\
       TM &  54865 & 55 & \citep{single-cell-7data}\\
       VISP &  15413 & 135 & \citep{single-cell-ALM-VISP} \\
        \bottomrule
    \end{tabular}
    \caption{Details of the single-cell RNA-seq datasets}
    \label{tab:dataset-details}
\end{table}

\paragraph{Bulk-RNA datasets:}
For the bulk-RNA datasets, we take patient samples from the The Cancer Genome Atlas~\cite{TCGA}. This was part of a comprehensive project launched by the ational Human Genome Research Institute (NHGRI) in 2006 to catalog genetic mutations responsible for cancer using genome sequencing and bioinformatics. The project analyzed samples of around 11,000 tumor samples across 33 cancer types (which we use as the ground-truth cluster identity). We use two modalities of datasets, mRNA and micro-RNA, which form the two bulk-RNA datasets that we use.

For the single-cell and bulk-RNA datasets, we first log-normalize it and then apply PCA dimensionality reduction of dimension $\min\{50,\text{\# of ground truth clusters}\}$, which is a standard pipeline in the single-cell (and genomics) analysis literature~\citep{single-cell-duo}.

\paragraph{Image datasets:}
For image datasets, we use the test-set of the popular MNIST~\cite{mnist-dataset} and Fashion-MNIST~\cite{fashion-mnist} dataset, as well as the popular CIFAR datasets. Following~\cite{shinde2024geometric}, we use the ViT embeddings for CIFAR 10, 20, and 100. Additionally, for CIFAR 10 we use both ViT-Large and ViT-Base embeddings.

\subsection{Comparison algorithms}

{\bf Density-based clustering}
As we have discussed, we focus on density based and manifold based clustering. For density based clustering, we implemented HDBSCAN~\cite{hdbscan}. We also implemented DBSCAN \cite{dbscan} and OPTICS \cite{ankerst1999optics} but were unable to use them as benchmarks, as they marked most points as outliers for most datasets. We also uesd three density peak clustering algorithms~\cite{DensityPeak,ADPClust,PECANN}.

{\bf Manifold-based clustering}
For Manifold clustering, we first looked at the scikit-learn spectral clustering~\cite{von2007tutorial} implementation. We found the default setting of rbf kernel took prohibitively long time and also did not produce good quality cluster, and therefore we did not include it in our benchmarks. In comparison, the K-nearest-neighbor-based kernel had much better performance, proving to be a competent benchmark. We also implemented two recent manifold clustering algorithms \cite{annalittle} \cite{trillos} with theoretical guarantees but the performances were suboptimal to spectral clustering by large margins, and thus were not reported in detail in this paper. It remains as future steps to analyze the underlying cause of their ineffectiveness on real world data.

{\bf Further comparisons}
Finally, we note that design of clustering algorithm is a very extensive field with new clustering algorithms being designed. Especially, both the areas of density-peak clustering and manifold clustering are very active. We aim to test our framework both on and against more such algorithms in future.

\subsection{Performance of CoreSPECT on small noiseless datasets}
\label{sec:small-datasets}

\begin{figure}[t]
    \centering
    \includegraphics[width=\linewidth]{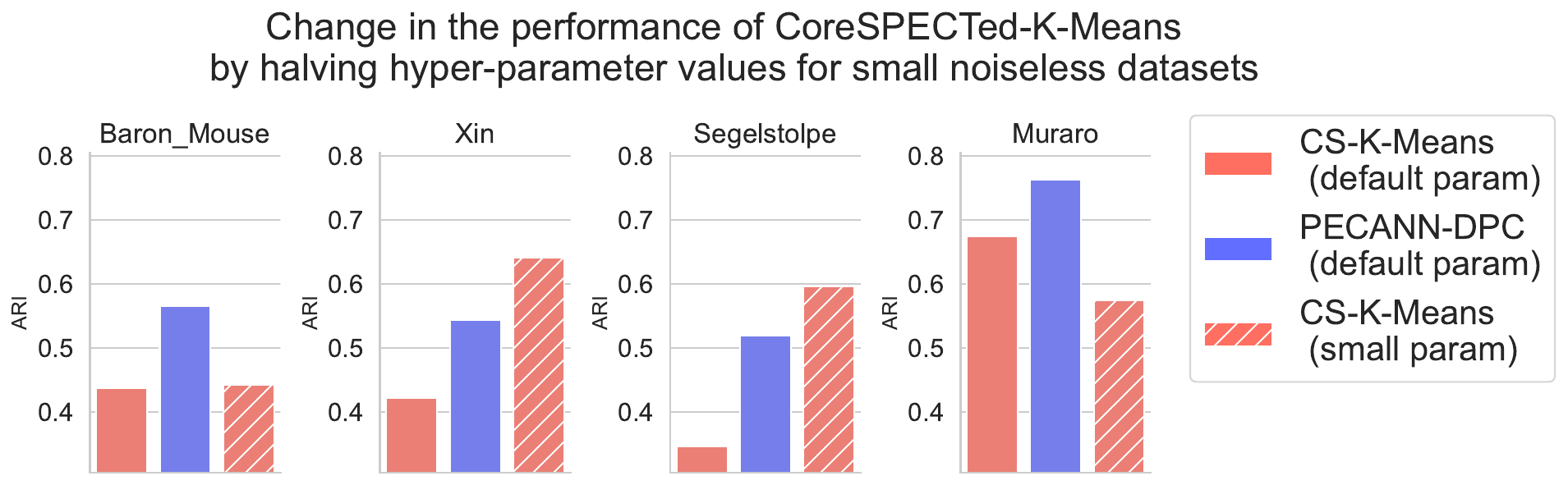}
    \caption{Enhancement in the performance of CoreSPECTed-K-Means by selecting smaller hyper-parameters for small noiseless datasets}
    \label{fig:small}
\end{figure}

In the main paper, we noted that in the four smallest datasets (Baron Mouse, Xin, Segerstolpe, and Muraro) our performance lacked behind that of Density based clustering (especially PECANN-DPC). Here we explore this in more detail. First, we note the ARI of CoreSPECTed-K-Means and PECANN-DPC for the four datasets. 

In fact, PECANN-DPC has noticeably higher ARI than our method for all four datasets. First we note that all of these clusters are well-separated. This can be observed as follows. If we look at any pair of clusters, all the nearest neighbors of any point are from the same cluster. This implies, the main hurdle of ``hard-to-separate-peripheral points'' is not present in these datasets. However, we still want to better understand if the performance of our framework can be improved in such cases. In this direction, we observe two interrelated reasons.

\begin{enumerate}
    \item For all of these datasets, the cores (top 10\% of the points) selected by FlowRank (with default parameter) sometimes does not include any points from the smallest (such as less than 20 points) clusters. 

    \item In our default setting, we choose $q=40,r=20,t=20$. The parameter $r$ decides how many neighbors of a point is looked at for the ascending random walk step. If $r$ is larger than the size of a ground truth cluster, then there is a high chance \emph{all the points} from that cluster will have a low FlowRank value. Similarly, in our layer-expansion step, we created a $t$-regular CDNN graph. Even if the  core has points from a small ground truth cluster that is well separated in the core-clustering step, a large $t$ can incorrectly assign non-core points from these small clusters.  
\end{enumerate}

Improving on the first point falls within the overarching direction of coming up with a better core-extraction algorithm, and we believe this is a very important question. With regards to the second point, we observe that our hyper-parameters can be changed slightly to get better results on these small datasets. As an example, we run CoreSPECTed-K-means with the hyper-parameters ($q=20, r=10,t=10$), that is dividing every value with $2$. The comparative performance of CS-K-Means w.r.t. the changed parameters are shown in Figure~\ref{fig:small}.

As we can observe, for 2 out of 4 datasets, we have considerable improvements (more than 50\%), only 2\% improvement for Baron-Mouse, and the ARI value for Muraro goes down by approximately 20\%. This demonstrates the potential of choosing parameters in a more principled way. 

However, this comparison may seem unfair as we are only using the default parameters for PECANN. In this direction, to understand the potential of the different clustering algorithms as well as our CoreSPECT framework, we report the best clustering performance for each dataset with the best hyper-parameters chosen from a grid search. Note that we do this in a supervised manner for all the clustering algorithms. The goal of this Section is to simply present the highest accuracy that can be achieved with the method, which is more of an indicator of how well the underlying mechanisms of the algorithms relate to/fit the underlying structures of the dataset.

\section{Limitations and Future Directions}
\label{app:limitations}

We conclude with a discussion on our limitations.

\begin{enumerate}[wide, labelindent=1pt]

%\item The main strength of the paper is also its primary limitation. We exploit a correlation in the geometry-density structure of the data, but this implies our framework is only applicable to such datasets. While we observe our framework is consistently applicable on large and poorly separated datasets, a notion for estimating when data fits our model is crucial for real-world application. 

\item Improving the core-extraction step is an important problem, as previously noted. We use the order provided by FlowRank and partition it into ten parts. Ideally, different data will have different numbers of layers for different clusters, and this phenomena needs to be explored further. We believe FlowRank is only a placeholder, and an improved algorithm can be achieved.

\item Next, we note that we have achieved the gains mentioned so far by fixing the other hyperparameters of our framework. However, we believe these hyperparameters can be more informatively selected that can further boost our performance. 

\item We have applied our framework to two algorithms. Applying and understanding the effect of our framework on different algorithms and different kind of embeddings is necessary to better understand the scope of application. 
 
    \item Data-based selection of the optimal hyper-parameter is an important direction. As we have observed, some of the optimal parameters depend on the internal geometry of the data. For example, the correct radius for choosing neighbors for ascending random walk depends on the width of the layers. A better inference of this underlying geometry may further raise confidence in our framework. 

    % \item In this paper, we have looked at four manifold clustering algorithm, but were able to only find one (the K-NN graph based spectral clustering) that provided competitive results. Comparing the performance of other manifold clustering algorithm is important to further understand the scope of the improvement our framework provides.  
    
\end{enumerate}

% Our framework and algorithms also stand to benefit from further theoretical insight. For example, we believe it will be possible to prove that FlowRank provides a layer-preserving ranking under our current assumptions. Similarly, better understanding the scope and limitations of the assumptions we make in the paper is also important. Here we note that the primary focus of this paper is not theoretical completion. Rather, we use theoretical insights of our model to create a generic framework which we apply to many datasets, on multiple methods, which we also compare with many clustering algorithms as well as provide various ablation studies that signify the importance and influence of the individual component of our framework. 

\end{document}